\documentclass[letterpaper]{article} 
\usepackage[arxiv]{aaai2026}  
\usepackage{times}  
\usepackage{helvet}  
\usepackage{courier}  
\usepackage[hyphens]{url}  
\usepackage{graphicx} 
\usepackage{natbib}  
\usepackage{caption} 
\usepackage[ruled,vlined,linesnumbered]{algorithm2e}
\usepackage{amsmath,amssymb,amsfonts,bm,amsthm}
\usepackage{booktabs}
\usepackage{multirow}
\usepackage{pifont}
\usepackage{subcaption}
\usepackage{enumitem}
\usepackage{graphicx}
\usepackage{newfloat}
\usepackage{listings}
\DeclareCaptionStyle{ruled}{labelfont=normalfont,labelsep=colon,strut=off} 
\newcommand{\cmark}{\ding{51}}
\newcommand{\xmark}{\text{\ding{55}}} 
\usepackage{hyperref} 

\title{\textsc{GraphOracle}: Efficient Fully-Inductive Knowledge Graph Reasoning via Relation-Dependency Graphs}

\author{
Enjun Du,
Siyi Liu,
Yongqi Zhang\thanks{~Corresponding author}
}
\affiliations{
The Hong Kong University of Science and Technology (Guangzhou)\\
\texttt{\{EnjunDu.cs, ssui.liu1022\}@gmail.com, yongqizhang@hkust-gz.edu.cn}
}

\usepackage{bibentry}
\begin{document}

\maketitle

\begin{abstract}
Knowledge graph reasoning in the fully-inductive setting—where both entities and relations at test time are unseen during training—remains an open challenge. In this work, we introduce \textsc{GraphOracle}, a novel framework that achieves robust fully-inductive reasoning by transforming each knowledge graph into a Relation-Dependency Graph (RDG). The RDG encodes directed precedence links between relations, capturing essential compositional patterns while drastically reducing graph density. Conditioned on a query relation, a multi-head attention mechanism propagates information over the RDG to produce context-aware relation embeddings. These embeddings then guide a second GNN to perform inductive message passing over the original knowledge graph, enabling prediction on entirely new entities and relations. Comprehensive experiments on 60 benchmarks demonstrate that \textsc{GraphOracle} outperforms prior methods by up to 25\% in fully-inductive and 28\% in cross-domain scenarios. Our analysis further confirms that the compact RDG structure and attention-based propagation are key to efficient and accurate generalization
\end{abstract}

\section{Introduction}

Knowledge graphs (KGs) encode structured knowledge as entity–relation–entity triples, serving as the backbone for scientific discovery, web-scale reasoning, and intelligent systems~\cite{Cai2025FusionKGLLM,Bai2025AutoSchemaKG,Dong2025GLMTripleGen,guo2025decouplingcontinualsemanticsegmentation,lin2025seagentselfevolutiontrajectoryoptimization,du2025graphmaster}. The central challenge in KG reasoning is \emph{link prediction}: given an incomplete graph and a query $(h, r, ?)$, predict the missing tail entity $t$~\cite{Du2025Mixture}. Inductive KG reasoning requires models to generalize to facts that were not explicitly observed during training. In the most challenging scenarios, models must perform \textbf{fully-inductive}~\cite{zhang2025trix,cui2024promptkg} reasoning—handling both unseen entities and relations during inference. \textbf{Cross-domain} generalization~\cite{wang2025mdgfm,hong2025stability,lin2025unifiedgnn} further requires reasoning over entirely different knowledge graphs containing 100\% novel entities and relations. Both settings pose a fundamental \emph{compositional generalization} challenge: models must recombine learned relational patterns to infer new facts in completely unfamiliar contexts.

Recent approaches to fully-inductive reasoning, including INGRAM~\cite{lee2023ingram} and ULTRA~\cite{galkin2024ultra}, attempt to address this challenge by constructing auxiliary relation graphs $G_R = (R, E_R)$ that capture transferable relational patterns. However, as knowledge graphs scale, this paradigm exposes two fundamental limitations that severely hinder their effectiveness in real-world applications. (1) These methods rely on co-occurrence statistics to connect relations, creating dense graphs with $|E_R| = \Theta(|R|^2)$ edges that inflate computational costs to $O(|R|^3 \cdot H)$ for $H$-layer message passing. The proliferation of spurious connections obscures meaningful compositional signals, while symmetric treatment of relation pairs erases the inherent directionality of logical composition. (2) Current models compress each relation into a single fixed embedding, forcing individual representations to capture diverse semantic roles across vastly different query contexts. For instance, the relation "associated with" may connect proteins to diseases in biomedical contexts but link authors to topics in academic graphs, yet existing methods use the same representation regardless of context.

These observations raise a critical research question: \emph{How to design a KG reasoning framework that captures meaningful relational dependencies while maintaining computational efficiency?} This requires moving beyond dense co-occurrence graphs to a sparse, directed structure that preserves compositional patterns, and replacing fixed relation embeddings with dynamic, query-dependent representations.


To realize this design, we propose \textbf{\textsc{GraphOracle}}, a relation-centric framework that transforms entity–relation interactions into a compact \textit{Relation-Dependency Graph (RDG)}. Unlike INGRAM and ULTRA, which generate excessive connections, our RDG retains only meaningful directed precedence links, yielding significantly fewer edges yet richer compositional patterns, with its effectiveness remaining stable as the number of relations increases. To capture relation dependencies for specific query, we design a multi-head attention mechanism that recursively propagates information over the RDG, dynamically assembling relation recipes conditioned on the query. By pre-training on four KGs in the general domain, \textsc{GraphOracle} only needs minimal finetune to achieve exceptional adaptability across transductive, inductive, and cross-domain reasoning tasks, improving performance by over 16.8\% on average compared to state-of-the-art methods. Our key contributions in this work can be summarized as follows:

\begin{itemize}[leftmargin=*]
    \item We introduce \textbf{\textsc{GraphOracle}}, a relation-centric foundation model that converts knowledge graphs into RDGs, explicitly encoding compositional patterns while reducing the number of edges on the relation graph compared to prior approaches.

    \item We develop a query-dependent multi-head attention mechanism that dynamically propagates information over the RDG, yielding domain-invariant relation embeddings that enable generalization to unseen entities, relations, and graphs.

    \item Extensive experiments across 60 benchmarks show that \textsc{GraphOracle} consistently outperforms SOTA methods, with particularly strong results in both fully-inductive and cross-domain settings, demonstrating its robustness and generalization capability in challenging scenarios.
\end{itemize}

\section{Related Works}

\subsubsection{Knowledge Graph Reasoning}
A Knowledge Graph (KG) consists of sets of entities $\mathcal{V}$, relations $\mathcal{R}$, and fact triples $\mathcal{F} \subseteq (\mathcal{V} \times \mathcal{R} \times \mathcal{V})$ as $\mathcal{G} = (\mathcal{V}, \mathcal{R}, \mathcal{F})$. 
$(e_q, r_q, e_a)$ is a triple in KG where $e_q, e_a \in \mathcal{V}$ and $r_{q} \in \mathcal{R}$. Knowledge graph reasoning encompasses several increasingly challenging settings based on what information is available during training versus inference. In the \textbf{transductive} setting, both entities and relations remain fixed: $(\mathcal{V}_{tra}=\mathcal{V}_{inf}) \land (\mathcal{R}_{tra}=\mathcal{R}_{inf})$. This allows models to learn fixed embeddings for all components. The \textbf{entity-inductive} setting introduces unseen entities at inference while keeping relations fixed: $(\mathcal{V}_{tra}\neq \mathcal{V}_{inf}) \land (\mathcal{R}_{tra} = \mathcal{R}_{inf})$. Most challenging is the \textbf{fully-inductive} setting where both entities and relations are novel: $(\mathcal{V}_{tra}\neq\mathcal{V}_{inf})\land(\mathcal{R}_{tra} \neq \mathcal{R}_{inf})$. Beyond these, \textbf{cross-domain} reasoning requires transferring to entirely different knowledge graphs with no shared entities or relations, demanding the most robust generalization capabilities.

\subsection{Transductive Reasoning}
Transductive methods assume all entities and relations at inference have been seen during training, enabling the use of fixed relation embeddings. Models like ConvE~\cite{ConvE}, RotatE~\cite{RotatE} and DuASE~\cite{li2024duase} learn low-dimensional relation and entity embeddings directly, while GNN variants such as R-GCN~\cite{schlichtkrull2018modeling} implement relation-specific message passing that effectively parameterizing relation influence via learned embedding-like transformations. These embedding-based approaches form strong baselines but fundamentally cannot generalize beyond their training vocabulary.

\subsection{Entity Inductive Reasoning}
Entity inductive KG reasoning relaxes the entity constraint while maintaining fixed relation embeddings. Early solutions leveraged auxiliary cues---text descriptions in content-masking models~\cite{shi2018open} or ontological features in OntoZSL~\cite{geng2021ontozsl}---and symbolic rule learners such as AMIE~\cite{galarraga2013amie} and NeuralLP~\cite{yang2017neural}. More recent approaches like DRUM~\cite{DRUM} employ differentiable rule chaining, while RLogic~\cite{RLogic} uses symbolic rule matching. SOTA GNN-based methods including GraIL~\cite{teru2020inductive}, PathCon~\cite{wang2021pathcon}, NBFNet~\cite{zhu2021nbfnet}, RED-GNN~\cite{Zhang2022RelationalDigraph}, A*Net~\cite{Zhu2023AStarNet} and AdaProp~\cite{Zhang2023AdaProp} propagate messages along relational paths to accommodate new entities---yet they still rely on fixed relation embeddings, limiting their applicability to scenarios with novel relations.

\subsection{Fully-Inductive Reasoning}
Fully-inductive settings demand handling both unseen entities and relations, requiring explicit relation graph structures. RMPI~\cite{geng2023relational} and INGRAM~\cite{lee2023ingram} pioneer this direction by constructing undirected relation graphs; however, RMPI is limited to subgraph extraction, and INGRAM's degree discretization hampers transfer across graphs with different relation distributions. ISDEA~\cite{gao2023double} and MTDEA~\cite{zhou2023ood} adopt double-equivariant GNNs, but their computational overhead restricts scalability. ULTRA~\cite{galkin2024ultra} advances this with interaction-conditioned relation graphs that adapt based on query context. TRIX~\cite{zhang2025trix} introduces expressive adjacency motifs for richer relation modeling, while KG-ICL~\cite{cui2024kgicl} employs prompt-based relation graphs.

\setlength{\tabcolsep}{1pt}
\begin{table}[!h]
\centering
\scriptsize
\begin{tabular}{l|@{\,}c@{\,}c@{\,}c p{3.3cm}}\toprule
Method & Ent-Ind. & Full-Ind. & Cross-Dom. & Relation Representation \\
\midrule
RotatE \& DuASE      & \xmark & \xmark & \xmark & Relation Embedding \\
A*Net \& AdaProp       & \cmark & \xmark & \xmark & Relation Embedding \\
DRUM       & \cmark & \xmark & \xmark & Differentiable rule chaining \\
RLogic       & \cmark & \xmark & \xmark & Symbolic rule matching \\
INGRAM      & \cmark & \cmark & \xmark & Undirected RG \\
ULTRA       & \cmark & \cmark & \xmark & Interaction-Conditioned RG \\
TRIX        & \cmark & \cmark & \xmark & Expressive Adjacency Motifs RG \\
KG-ICL      & \cmark & \cmark & \xmark & Prompt RG \\
\hline
\textsc{GraphOracle} & \cmark & \cmark & \cmark & Relation-Dependency Graph \\
\bottomrule
\end{tabular}
\caption{Comparison of inductive capabilities and relation representations. ``RG'' is short for ``relation graph''.}
\label{tab:indu_reletion}
\end{table}

\subsection{Cross-domain Reasoning}
Cross-domain KG reasoning represents the frontier of generalization, transferring patterns to graphs with entirely new entities and relations. Early work relied on domain-agnostic logical rules; recent advances leverage pre-trained graph foundation models. MDGFM~\cite{wang2025mdgfm} introduces multi-domain contrastive pre-training, while SAMGPT~\cite{zhang2025samgpt} demonstrates strong transfer without textual signals. Stability-GNN~\cite{hong2025stability} addresses structural shift through adversarial perturbations, and UnifiedGNN~\cite{lin2025unifiedgnn} jointly handles multiple inductive settings via relation adapters. RiemannGFM~\cite{liu2025riemanngfm} incorporates geometric regularization, while Text-Free MDGPT~\cite{li2025mfgpt} and GraphMFM~\cite{cheng2024graphmfm} scale cross-domain pre-training through modality-agnostic masked modeling.

As summarized in Table~\ref{tab:indu_reletion}, KG reasoning methods progress from fixed relation embeddings (transductive) to rule-based reasoning (entity-inductive) to explicit relation graphs (fully-inductive). While existing fully-inductive methods construct undirected or interaction-conditioned graphs, they remain limited to single-domain scenarios. \textsc{GraphOracle} uniquely introduces directed relation-dependency graphs that capture compositional patterns, enabling the first successful cross-domain generalization.

\begin{figure*}[t]
  \centering
    \centering
    \includegraphics[width=0.9\textwidth]{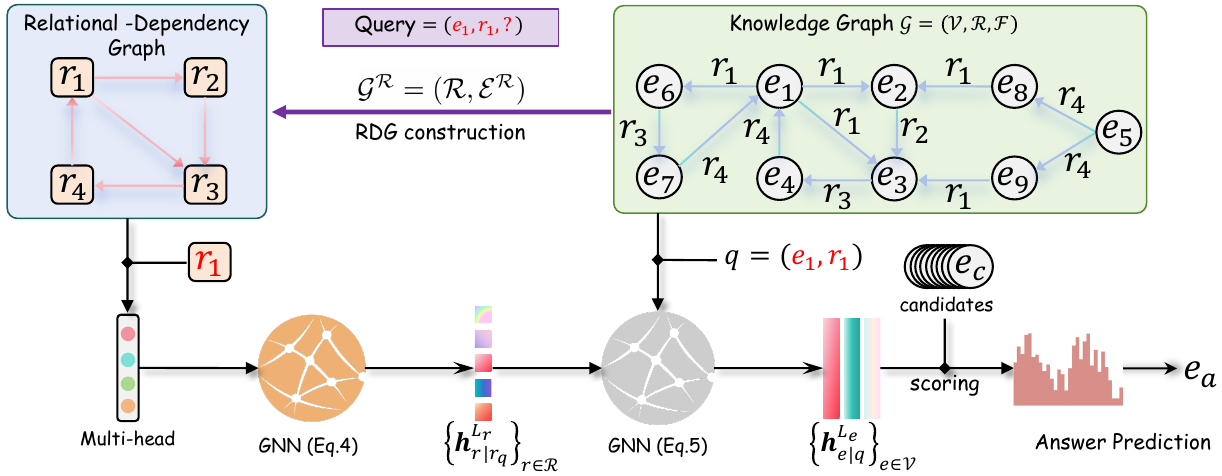}
  \caption{Overview of the \textsc{GraphOracle} process that predicts the answer entity $e_a$ from a given query $(e_1, r_1, ?)$: Given a Knowledge Graph, we first construct the Relation Dependency Graph (RDG). Then, a multi-head attention mechanism combined with a GNN is used to propagate messages among RDG to obtain relation representations $\bm h_{r|r_q}^{L_r}$, which are then used in another GNN for message passing over entity representations. Finally, the candidate entities are scored and evaluated based on the aggregated entity representations,
  and then
  ranked for answer entity prediction.}
  \label{fig:main_figure}
\end{figure*}

\section{Preliminary}

Given a query with missing answer
$(e_q, r_q, ?)$, the goal is to find an answer entity $e_a$ such that $(e_q, r_q, e_a)$ is true. 
Most state-of-the-art models 
leverage GNN to aggregate relational paths
and
can be formulated as the following recursive function, 
where each candidate entity $e_y$ at step $\ell$ accumulates information from its in-neighbors:
\begin{equation}
\bm h_{r_q}^{\ell}(e_q, e_y) = \bigoplus_{(e_x, r, e_y) \in \mathcal{N}(e_y)} \!\! \bm h_{r_q}^{\ell-1}(e_q, e_x) \otimes \phi(r, r_q),
\label{eq:gnn}
\end{equation}
where $\bm h_{r_q}^{0}(e_q, e) = \bm{1}$ if $e = e_q$, and $\bm{0}$ otherwise, $\oplus/\otimes$ are learnable additive and multiplicative operators. $\phi$ encodes relation-type compatibility. After $L$ steps' iteration, the answer is ranked by the score
$
s(e_q, r_q, e_a) = \bm{w}_s^{\top}\bm h_{r_q}^{L}(e_q, e_a)
$.
The learning objective is formulated in a contrastive approach that maximizes the log-likelihood of correct triples in the training set, which amounts to minimizing:

\begin{equation}
\small
\begin{aligned}
\mathcal{L}_{\text{train}} = - \!\!\!\!\!  \!  \!  \!  \! \sum_{(e_q, r_q, e_a) \in \mathcal{F}_{\text{train}}} &\Big[ \log \sigma\big(s(e_q, r_q, e_a)\big) \\
\! +&  \!\!\!\!\!\!\!\!\sum_{e'_n \in \mathcal{N}_{(e_q, r_q, e_a)}}  \!\!\!\!\!\!\!\!\log\big(1 - \sigma(s(e_q, r_q, e'_n))\big)
\Big],
\end{aligned}
\label{eq:loss}
\end{equation}
where $\sigma$ is a sigmoid function.
\textsc{GraphOracle} retains this framework and introduces a Relation-Dependency Graph  pre-training objective that endows $\phi$ with universal semantics, enabling zero-shot generalization to \emph{both} unseen entities and unseen relation vocabularies.


\section{The Proposed Method}



In order to enable fully inductive KG reasoning
and improve the generalization ability of models across KGs,
the key is to generalize the dependencies among relations for different KGs.
To achieve this goal,
we firstly introduce Relational Dependency Graph (RDG),
which explicitly models how relations depend on each other,
in Section~\ref{ssec:rdg}
Based on RDG, we propose a query-dependent multi-head attention mechanism to learn relation representations 
from a weighted combination of precedent relations
in Section~\ref{ssec:relationembed}.
Subsequently, in Section~\ref{ssec:entityembed}, 
we introduce the approach where
entity representations are represented with
the recursive function~\eqref{eq:gnn}
in the original KGs by using the relation representations just obtained.
The overview of our approach is shown in Fig~\ref{fig:main_figure}.


\subsection{RDG Construction}
\label{ssec:rdg}

To build an effective KGFM capable of cross KG generalization, 
we must capture the fundamental dependent patterns through which one relation can be represented by a combination of others. 
Our key innovation is reparameterizing entity–relation interactions as a relation-dependency interaction structure that explicitly captures how relations influence each other.


Given a KG $\mathcal{G}=(\mathcal{V},\mathcal{R},\mathcal{F})$, 
we construct a RDG $\mathcal{G}^{\mathcal{R}}$ through a
structural transformation. 
First, we define a relation adjacency operator 
$\Phi: \mathcal{F} \rightarrow \mathcal{R} \times \mathcal{R}$
that extracts transitive relation dependencies:

\begin{equation*}
\small
\Phi(\mathcal{F})
=\! \! \! \! \! \!\bigcup\limits_{e,e',e'' \in \mathcal{V}}\! \! \!  \!
  \{(r_i,r_j)\mid (e,r_i,e')\in\mathcal{F}\land(e',r_j,e'')\in\mathcal{F}\}.
\end{equation*}

Then, the RDG is defined as $\mathcal G^\mathcal R =(\mathcal{R}, \mathcal{E}^{\mathcal{R}})$,
where 
the node set $\mathcal R$ contains all the relations
and the edge sets
$\mathcal{E}^{\mathcal{R}} = \Phi(\mathcal{F})$ includes relation dependencies induced by entity-mediated pathways.
This transformation alters the 
conceptual framework, 
shifting from an entity-centric perspective to a relation-interaction manifold 
where 
compositional connections between relations become explicit. Each directed edge $(r_i, r_j)$ in $\mathcal{G}^{\mathcal{R}}$ represents a potential relation-dependency pathway, indicating that relation $r_i$ preconditions relation $r_j$ through their sequential interaction over a shared entity context. The edge structure encodes compositional relational semantics, capturing how one relation may influence the probability or applicability of another when they occur in sequence.

To incorporate the 
hierarchical and compositional nature of relation interactions, we define a partial ordering function $\tau: \mathcal{R} \rightarrow \mathbb{R}$ that assigns each relation a position in a relation precedence structure.
This ordering is derived from the KG's inherent structure through rigorous topological analysis of relation co-occurrence patterns and functional dependencies. Relations that serve as logical precursors in inference chains are assigned lower $\tau$ values, thereby establishing a directed acyclic structure in the relation graph that reflects the natural flow of information propagation.
Using this ordering, we define the set of preceding relations for any relation $r_v$ as:
\begin{equation}
\!\!\!\mathcal{N}^{\text{past}}(r_v) \!=\! \{r_u \!\in\! \mathcal{R} \mid \!(r_u, r_v) \!\in\! \mathcal{E}^{\mathcal{R}} \text{ and } \tau(r_u) \!<\! \tau(r_v)\}.\!\!
\label{eq:neighbor}
\end{equation}
This formulation enables us to capture the directional dependency patterns where relations with lower positions in the hierarchy systematically precede and inform relations with higher $\tau$. By explicitly modeling these precedence relationships, our framework can identify and leverage compositional reasoning patterns that remain invariant across domains, enhancing the generalization capabilities.

\subsection{
Relation Representation Learning on RDG}
\label{ssec:relationembed}

Building on the constructed RDG $\mathcal{G}^{\mathcal{R}}$, we develop a representation mechanism that captures the contextualized semantics of relations conditioned on a specific query. 
Given a query relation $r_q$, we introduce an RDG aggregation mechanism to compute $d$-dimensional relation-node representations $\bm{R}_q \in \mathbb{R}^{|\mathcal{R}| \times d}$ conditional on $r_q$.

Following Eq.~\eqref{eq:gnn}.
we apply a labeling initialization 
to distinguish the query relation node $r_q$ in $\mathcal{G}^{\mathcal{R}}$.
Then employ multi-head attention relation‑dependency message passing over the graph:
\begin{equation}
\small
\begin{aligned}
& \bm{h}_{r_v \mid r_q}^0 = \texttt{INDICATOR}_r(r_v, r_q) 
= \delta_{r_v, r_q} \cdot \bm{1}^d, 
\quad r_v \in \mathcal{G}^{\mathcal{R}} \\[0.5em]
& \bm{h}^{\ell}_{r_v \mid r_q} = 
\sigma\Big( \frac{1}{H} \sum_{h=1}^H \Big[ \quad\; \bm{W}_1^{\ell,h} \!\!\!\!\sum_{r_u \in \mathcal{N}^{\text{past}}(r_v)} 
\hat \alpha_{r_u r_v}^{\ell,h} \, \bm{h}^{\ell-1}_{r_u \mid r_q} \\
& \quad \quad  \quad  \quad  \quad   + 
\bm{W}_2^{\ell,h} \, \hat \alpha_{r_v r_v}^{\ell,h} 
\, \bm{h}^{\ell-1}_{r_v \mid r_q}
\Big] \Big),
\end{aligned}
\label{eq:head}
\end{equation}
where $\delta_{r_v,r_q} = 1$ if $v = q$, and $0$ otherwise. 
$H$ is the number of attention heads, 
and $\bm W^{\ell,h}_1, \bm W^{\ell,h}_2\in \mathbb{R}^{d \times d}$ are head-specific parameter matrices.
The relation‑dependency attention weight 
$\hat{ \alpha}_{r_u r_v}^{\ell,h}$ captures the directional influence of relation $r_u$ on relation $r_v$, computed as:

\begin{equation*}
\hat{ \alpha}_{r_u r_v}^{\ell,h} = \frac
{\exp \big (\bm a^T (\bm W^h_\alpha \bm h_{r_u}^{\ell-1} \| \bm W_\alpha^h \bm h_{r_v}^{\ell-1})\big)}
{\sum_{r_w \in \mathcal{N}^{\text{past}}(r_v)} \exp \big(\bm a^T (\bm W_\alpha^h \bm h_{r_w}^{\ell-1} \| \bm W_\alpha^h \bm h_{r_v}^{\ell-1})\big)},
\label{alg:alpha}
\end{equation*}
where $\bm a \in \mathbb{R}^{2d}$ is a learnable attention parameter vector, $\|$ denotes vector concatenation, and $\bm W^h_\alpha \in \mathbb{R}^{d \times d}$ are head-specific trainable projection matrix. 
The neighborhood function $\mathcal{N}^{\text{past}}(r_v)$ enforces the relation‑dependency ordering of relations as defined in Eq.~\eqref{eq:neighbor}.

After $L_r$ layers of message passing, the final relation representation incorporates both local and higher-order dependencies $\bm R_q = \{\bm h_{r|r_q}^{L_r} \mid r \in \mathcal{R}\}$.

\subsection{Entity Representation Learning on the Original KG}
\label{ssec:entityembed}

After obtaining the relation representations $\bm R_q$ from RDG
conditioned on $r_q$, we obtain query-dependent entity representations 
by conducting message passing over the original KG structures.
This approach enables effective reasoning across both seen and unseen entities and relations.

For a given query $(e_q, r_q, ?)$, we compute entity representations recursively with Eq.~\eqref{eq:gnn} through the KG $\mathcal G$. 
The initial representations
$\bm h^{0}_{e|q}=\bm 1$ 
if $e=e_q$, and otherwise $\bm 0$.
At each layer $\ell$, the representation of an entity $e$ is computed as:
\begin{equation}
\begin{array}{ll}
\!\!\!\!\bm h^{\ell}_{e|q} = \delta \!\left( \bm W^{\ell} \!\cdot\! \sum_{(e_s, r, e) \in \mathcal F_\text{train}} \!\alpha^{\ell}_{e_s, r|q} \!\big( \bm h^{\ell-1}_{e_s|q} + \bm h^{L_r}_{r|r_q} \big) \!\right)\!,\!\!
\end{array}
\label{eq:gnnentity}
\end{equation}
where 
$\delta(\cdot)$ is a non-linear activation, and
the attention weight $\alpha^{\ell}_{e_s, r|q}$  is computed as:
\begin{equation*}
\alpha_{e_s,r|q}^\ell = \sigma\Big((\bm w_\alpha^\ell)^\top \text{ReLU}\big(\bm  W_\alpha^\ell \cdot (\bm h^{\ell-1}_{e_s|q} \| \bm h^{L_r}_{r|r_q} \| \bm h^{L_r}_{r_q|r_q})\big)\Big).
\label{eq:alhpa}
\end{equation*}
where 
$\bm w_\alpha^\ell \in \mathbb{R}^d$ and 
$\bm W_{\alpha}^\ell \in \mathbb{R}^{d \times 3d}$ are learnable parameters, $\sigma$ is the sigmoid function and $\cdot$ denotes the standard matrix-vector multiplication.

We iterate Eq.~\eqref{eq:gnnentity}
for $L_e$ steps and use the final layer representation $\bm h_{e|q}^{L_e}$ for scoring each entity $e\in\mathcal V$.
The critical idea here is replacing the learnable relation embeddings $\bm r$ with 
the contextualized relation embedding $\bm h^{L_r}_{r|r_q}$ from our RDG,
enabling fully inductive reasoning (Time complexity  of the \textsc{GraphOracle} model is given in Appendix C,
and the
Theoretical analysis on its expressiveness and generalization is given in Appendix I).

\begin{figure*}[h!]
  \centering
    \centering
    \includegraphics[width=\textwidth]{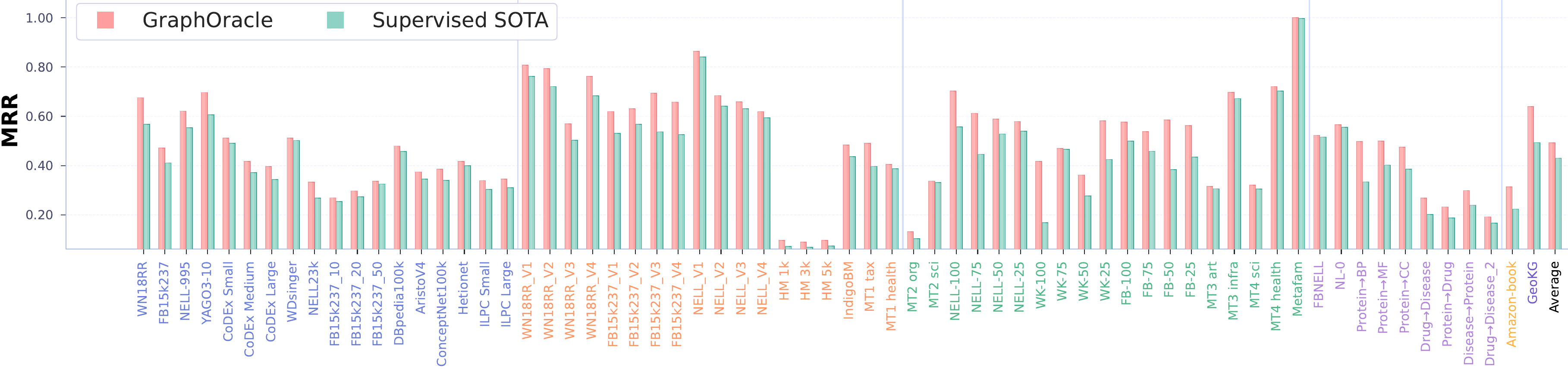}
  \caption{Comparison of the MRR performance (the larger the better) between \textsc{GraphOracle} and supervised SOTA methods across various datasets. 
          Note that Amazon-book uses NDCG@20 due to its adaptation to the recommendation task.}
  \label{fig:main_compare}
\end{figure*}


\subsection{Training Details}

All the learnable parameters
such as $\bigl\{
\bm W^{h}_{O},\;
\bm W^{\ell,h}_{1},\;
\bm W^{\ell,h}_{2},\;
\bm W^{h}_{\alpha},\;
\bm a,\;
\bm W^{\ell},\;
\bm W^{\ell}_{\alpha},\;
\bm w^{\ell}_{\alpha},\;
\bm w^{L}
\bigr\}$
are trained end-to-end by minimizing the loss function Eq.~\eqref{eq:loss}.
\textsc{GraphOracle} adopts a sequential multi‑dataset pre‑train $\rightarrow$ fine‑tune paradigm to 
acquire a general relation-dependency graph representation across KGs $\{\mathcal{G}_1,\dots,\mathcal{G}_K\}$.  
For each graph $\mathcal{G}_k$, we optimize the regularized objective
$
\mathcal{L}^{(k)}
    = \mathcal{L}_{\text{task}}^{(k)}
    + \lambda_k \,\bigl\lVert\Theta\bigr\rVert_2^{\;2},
$
where $\mathcal{L}_{\text{task}}^{(k)}$ denotes the task‑specific loss on $\mathcal{G}_k$ (e.g., Eq.~\eqref{eq:loss}),  
$\Theta$ represents all learnable parameters, and $\lambda_k$ controls the strength of L2 regularization.
Early stopping technique is used for each graph
once validation MRR fails to improve for several epochs.
This iterative pre‑train process, together with our relation‑dependency graph encoder, equips \textsc{GraphOracle} with strong cross‑domain generalization.
When adapting \textsc{GraphOracle} to new KGs, 
we firstly build the RDG
and then support two inference paradigms:
\begin{itemize}[leftmargin=*]
    \item \textbf{Zero-shot Inference}. 
    The pre-trained model is directly applied to unseen KGs without tuning. 
    \item \textbf{Fine-tuning}. 
    For more challenging domains, 
    we fine-tune the pre-trained parameters
    on the target KG $\mathcal{G}_{\text{target}}$ for a limited number of epochs $E_{\text{fine-tune}} \ll E_{\text{train}}$.
\end{itemize}

\begin{table*}[htbp]
\centering
\renewcommand{\arraystretch}{1.2}
\small
\setlength{\tabcolsep}{3pt}
\begin{tabular}{@{}lcccccccccccc@{}}
\toprule
\textbf{Model} & \multicolumn{3}{c}{\textbf{Transductive}} & \multicolumn{3}{c}{\textbf{Entity Inductive}} & \multicolumn{3}{c}{\textbf{Fully Inductive}} & \multicolumn{3}{c}{\textbf{Cross-domain}} \\
\cmidrule(lr){2-4} \cmidrule(lr){5-7} \cmidrule(lr){8-10} \cmidrule(lr){11-13}
\textbf{Metric} & MRR & H@1 & H@10 & MRR & H@1 & H@10 & MRR & H@1 & H@10 & MRR & H@1 & H@10 \\
\midrule
Supervised SOTA & 0.4185 & 0.4715 & 0.5771 & 0.4915 & 0.4730 & 0.6296 & 0.4593 & 0.2942 & 0.6246 & 0.2964 & 0.2149 & 0.4458 \\
GraphOracle     & 0.4486 & 0.5550 & 0.6111 & 0.5449 & 0.5684 & 0.6722 & 0.5203 & 0.3688 & 0.7279 & 0.3759 & 0.2744 & 0.5485 \\
Improvement     & 7.19\% & 17.70\% & 5.89\% & 10.86\% & 20.16\% & 6.77\% & 13.28\% & 25.36\% & 16.54\% & 26.82\% & 27.70\% & 23.03\% \\
\bottomrule
\end{tabular}
\caption{Average performance comparison between \textsc{GraphOracle} and Supervised SOTA under four generalization settings.}
\label{tab:improve}
\end{table*}

\begin{figure*}[h]
  \centering
  \begin{minipage}[t]{0.49\textwidth}
    \centering
    \includegraphics[width=\linewidth]{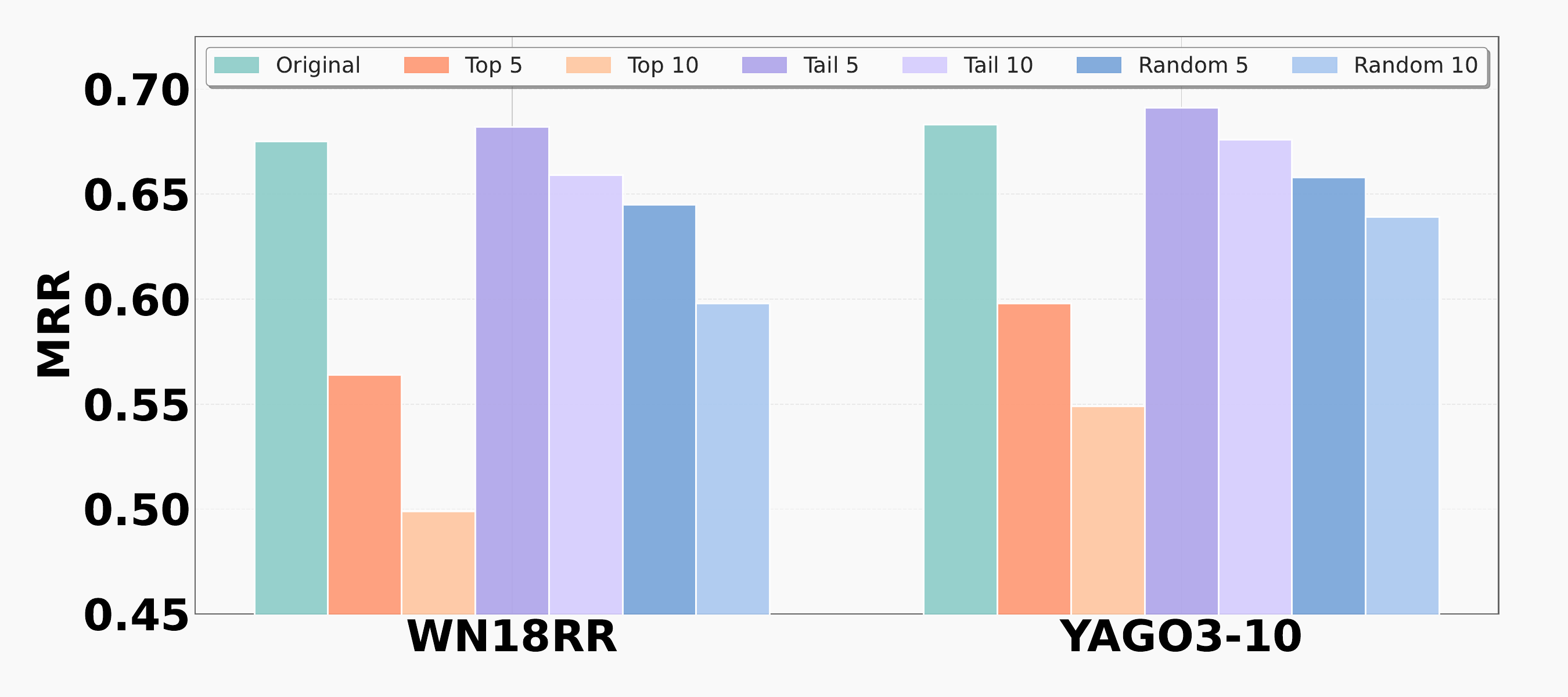}
    \caption{Perturbation Analysis of RDG Edges by Attention-Derived Importance Scores.}
    \label{fig:relation_Explainability}
  \end{minipage}
  \hfill
  \begin{minipage}[t]{0.49\textwidth}
    \centering
    \includegraphics[width=\linewidth]{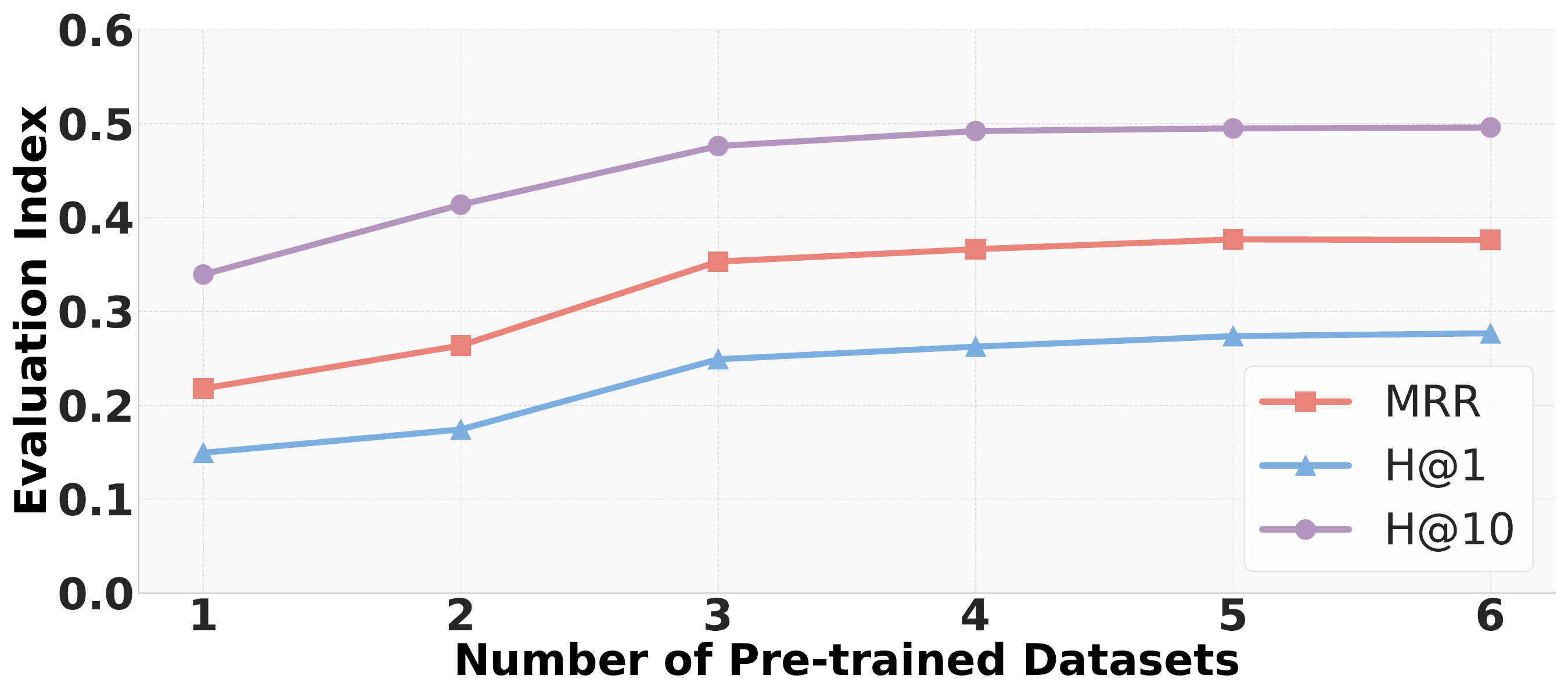}
    \caption{Impact of Number of Pre-trained Datasets on Zero-Shot Evaluation Metrics.}
    \label{fig:ablation_figure}
  \end{minipage}
  
\end{figure*}

\section{Experiment}

To evaluate the comprehensive capabilities of \textsc{GraphOracle} as a Foundation Model for KG reasoning, we formulate the following research questions:
\textbf{RQ1}:
How does \textsc{GraphOracle} model perform compared with state-of-the-art models  on diverse KGs and cross-domain datasets?
\textbf{RQ2}:
How do different relation-dependency patterns contribute to \textsc{GraphOracle}'s performance?
\textbf{RQ3}: To what extent can external information enhance the performance of \textsc{GraphOracle}?
\textbf{RQ4}: How do the components and configurations contribute to the performance?

\subsection{Experimental Setup} \label{Experimental Setup}

\subsubsection{Datasets} 
We conduct comprehensive experiments on $60$ KGs, which we classify into three categories according to their properties (Details are given in Appendix D.):
\begin{itemize}[leftmargin=*]
    
    \item \textbf{Transductive and Inductive datasets.} To ensure fair comparison, we follow the same dataset settings as ULTRA~\cite{galkin2024ultra}, TRIX~\cite{zhang2025trix}, and KG-ICL~\cite{cui2024promptkg}, including 16 transductive, 18 entity-inductive, and 23 fully-inductive datasets, totaling 57 in all.

    \item \textbf{Cross domain datasets.} (i) \textbf{Biomedical Datasets}: We use biomedical KG PrimeKG~\cite{chandak2023building} to examine the cross-domain capabilities of \textsc{GraphOracle}.
    We finetune with 80\% samples in raw PrimeKG, and validate with 10\% samples.
    When testing on the remaining 10\%, we specially focus on the predictions for triplets: (\textrm{Protein}, \textit{Interacts\_with}, \textrm{BP/MF/CC}), (\textrm{Drug}, \textit{Indication}, \textrm{Disease}), (\textrm{Drug}, \textit{Target}, \textrm{Protein}), (\textrm{Protein}, \textit{Associated\_with}, \textrm{Disease}), (\textrm{Drug}, \textit{Contraindication}, \textrm{Disease}). (ii) \textbf{Recommendation domain}: We transform the Amazon-book~\cite{wang2019kgat} dataset into a pure KG reasoning dataset to adapt to the KGs Reasoning field by defining the interactions between users and items as a new relation in the KG. (iii) \textbf{Geographic datasets} (GeoKG)~\cite{geonames} . (Detail process are given in Appendix F.)
\end{itemize}

\subsubsection{Pretrain and Finetune.} 
\textsc{GraphOracle} is pre-trained on three general KGs (NELL-995, CoDEx-Medium, FB15k-237) to capture diverse relational structures and reasoning patterns.
It takes 150,000 training steps with batch size of 32 using AdamW optimizer~\cite{loshchilov2019decoupled} 
on a single A6000 (48GB) GPU. For cross-domain adaptation, we employ a lightweight fine-tuning approach that updates only the final layer parameters while freezing the pre-trained representations. 
The finetune process only takes $1\sim2$ epochs to achieve the best results. The pre-training process takes approximately 36 hours, while fine-tuning requires only $15\sim 60$ minutes depending on the target dataset. Detailed hyperparameters, architecture specifications, and training configurations are provided in Appendix E.

\subsubsection{Baselines} 
We compare the proposed \textsc{GraphOracle} with
(i) \textbf{Transductive:} ConvE \cite{ConvE}, QuatE \cite{QuatE}, DuASE~\cite{li2024duase} and BioBRIDGE~\cite{wang2024biobridge}; 
(ii) \textbf{Entity inductive:} MINERVA \cite{MINERVA}, DRUM \cite{DRUM}, AnyBURL~\cite{meilicke2020reinforced}, RNNLogic \cite{RNNLogic}, RLogic \cite{RLogic} GraphRulRL~\cite{mai2025graphrulrl}, CompGCN \cite{CompGCN}, NBFNet \cite{zhu2021neural}, RED-GNN \cite{RED-GNN}, A*Net \cite{Zhu2023AStarNet}, Adaprop \cite{Adaprop} and one-shot-subgraph \cite{one-shot-subgraph}; 
(iii) \textbf{Fully inductive:} INGRAM~\cite{lee2023ingram}, ULTRA~\cite{galkin2024ultra}, TRIX~\cite{zhang2025trix} and KG-ICL~\cite{cui2024promptkg}. The results for the baseline methods were either directly obtained from the original publications or reproduced using the official source code provided by the authors. Due to page limitations, some other baselines can be found in INGRAM~\cite{lee2023ingram}, BioBRIDGE~\cite{wang2024biobridge} and KUCNet~\cite{liu2024kucnet}.

\subsection{Overall Performance (\textbf{RQ1})}
\label{section:rq1}

\begin{figure*}[h!]
  \centering
    \centering
    \includegraphics[width=0.83\textwidth]{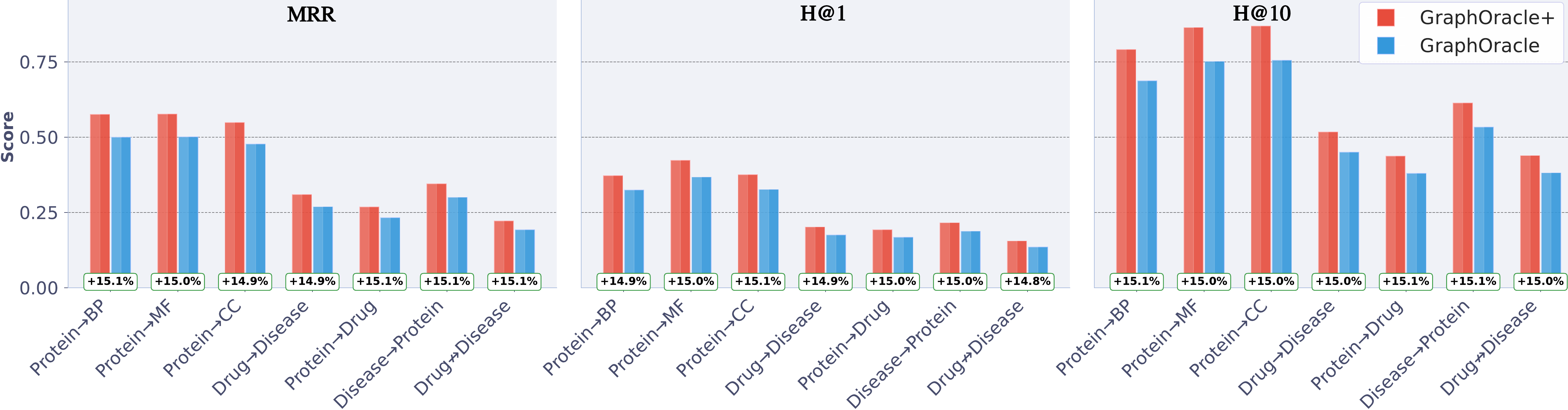}
  \caption{Comparison on PrimeKG: Evaluating \textsc{GraphOracle} Enhanced by External Entity information (\textsc{GraphOracle+}).}
  \label{fig:GraphOracle+}
\end{figure*}

\begin{table*}[h]
\centering

\setlength{\tabcolsep}{1mm}
\small
\begin{tabular}{l|ccc|ccc|ccc|ccc|ccc}
\hline
\multirow{2}{*}{Models} 
& \multicolumn{3}{c|}{Nell-100} 
& \multicolumn{3}{c|}{WK-100} 
& \multicolumn{3}{c|}{FB-100} 
& \multicolumn{3}{c|}{YAGO3-10} 
& \multicolumn{3}{c}{GeoKG} \\
& MRR & H@1 & H@10 
& MRR & H@1 & H@10 
& MRR & H@1 & H@10 
& MRR & H@1 & H@10 
& MRR & H@1 & H@10 \\
\hline
\textsc{GraphOracle} & \textbf{0.702} & \textbf{0.623} & \textbf{0.905} 
& \textbf{0.417} & \textbf{0.192} & \textbf{0.711} 
& \textbf{0.576} & \textbf{0.407} & \textbf{0.812} 
& \textbf{0.696} & \textbf{0.672} & \textbf{0.807} 
& \textbf{0.639} & \textbf{0.552} & \textbf{0.793} \\
\hline
W/o RDG & 0.376 & 0.278 & 0.493 
& 0.132 & 0.102 & 0.243 
& 0.237 & 0.149 & 0.432 
& 0.593 & 0.545 & 0.712 
& 0.521 & 0.452 & 0.614 \\
W/o multi-head & 0.598 & 0.534 & 0.817 
& 0.302 & 0.39 & 0.619 
& 0.492 & 0.321 & 0.724 
& 0.629 & 0.576 & 0.722 
& 0.566 & 0.477 & 0.646 \\
\hline
$\text{Graph}_{\text{INGRAM}}$ & 0.478 & 0.375 & 0.663 & 0.189 & 0.092 & 0.389 & 0.398 & 0.283 & 0.557 & 0.478 & 0.512 & 0.625 & 0.485 & 0.471  &  0.613\\
$\text{Graph}_{\text{ULTRA}}$ & 0.548 & 0.490 & 0.720 & 0.201 & 0.105 & 0.466 & 0.465 & 0.297 & 0.679 & 0.587 & 0.583 & 0.725 & 0.545 & 0.491 & 0.645 \\
$\text{Message}_{\text{INGRAM}}$ & 0.517 & 0.426 & 0.726 & 0.271 & 0.111 & 0.518 & 0.428 & 0.289 & 0.635 & 0.539 & 0.473 & 0.674 & 0.496 & 0.481 & 0.624 \\
$\text{Message}_{\text{ULTRA}}$ & 0.569 & 0.502 & 0.741 & 0.282 & 0.124 & 0.597 & 0.468 & 0.305 & 0.672 & 0.503 & 0.539 & 0.669 & 0.457 & 0.479 & 0.636 \\
\hline
\end{tabular}
\caption{Ablation Analysis of \textsc{GraphOracle}'s Core Architectural Components across Five Benchmark Datasets.}
\label{tab:ablation}
\end{table*}

The main experimental results, presented in Fig.~\ref{fig:main_compare}, illustrate the performance of \textsc{GraphOracle} following pre-training on three KGs and a brief fine-tuning phase of just two epochs across 60 distinct datasets (comprehensive results are available in Appendix G). A salient finding is that \textsc{GraphOracle} consistently outperforms the supervised SOTA across all evaluated baseline datasets and metrics, as detailed in Table~\ref{tab:improve}. This robust performance underscores the overall efficacy of our methodology, with particularly notable improvements observed in the more challenging scenarios. The results demonstrate substantial gains across all reasoning types, with the most pronounced improvements occurring in fully inductive and cross-domain settings, where the model must generalize to entirely unseen entities or domains—scenarios that represent the most stringent tests of a model's reasoning capabilities.

\subsection{Relation-Dependency Pattern Analysis (\textbf{RQ2})}
\label{section:rq2}

To investigate whether \textsc{GraphOracle} truly internalizes the compositionality of the relations, that is, the way complex relations are constructed systematically from simpler ones, we performed a series of perturbation analyzes on the learned RDG. First, we calculated relation attention weights using Eq.~\eqref{alg:alpha}, averaged over the WN18RR and YAGO3-10 datasets, to assign an \emph{importance score} to each relation pair, reflecting its contribution to downstream reasoning. Subsequently, during inference, we systematically disabled specific subsets of edges based on these attention weights: (i) the top-5 and top-10 most important (highest attention) compositional relation pairs; (ii) the bottom-5 and bottom-10 least important (lowest attention) pairs; and (iii) 5 and 10 randomly selected pairs.

As shown in Fig.~\ref{fig:relation_Explainability}, removing the highly-ranked compositional edges—those encoding key multi-hop templates essential for composing higher-level relations—causes a sharp decline in MRR on both WN18RR and YAGO3-10.  This confirms that \textsc{GraphOracle} heavily relies on these identified compositional pathways for its predictions. Conversely, suppressing a small number of low-importance edges sometimes leads to slight performance improvements, suggesting that these weaker compositional cues might act as semantic noise. Perturbations involving randomly removed edges result in only moderate performance degradation. This underscores the idea that it is not merely the quantity of relations but the specific, learned compositional interactions between them that are crucial for \textsc{GraphOracle}'s reasoning process. These findings collectively substantiate that \textsc{GraphOracle}'s predictions are rooted in the compositional structure of its RDG, rather than relying on isolated relation statistics.


\subsection{Compatible with Additional Initial Information (\textbf{RQ3})}
\label{section:rq3}

To explore the potential of external information in enhancing KG reasoning, we introduced an improved entity initialization strategy. This involved incorporating modality-specific encoded features as initial entity vectors, moving beyond standard random initialization. The resulting model, denoted as \textbf{\textsc{GraphOracle+}}, leverages foundation model embeddings to create more semantically rich entity representations (details are provided in Appendix H). As demostrated in Fig~\ref{fig:GraphOracle+} (details are given in Tabel~\ref{tab:GraphOracle_pro}), experimental evaluations on the PrimeKG dataset show that \textsc{GraphOracle+} achieves consistent performance gains across all metrics. Notably, MRR scores improved by 15\% for protein-biological process prediction and 14\% for protein-molecular function prediction. These results affirm that \textsc{GraphOracle}'s framework significantly benefits from integrating external information. In an era increasingly influenced by large language models, the capacity for flexible incorporation of diverse information sources is crucial for advancing generalization and adaptability, especially within specialized and complex domains such as biomedicine.


\subsection{Ablation Study (\textbf{RQ4})}

\label{section:rq4}


To rigorously evaluate the contribution of each architectural component within \textsc{GraphOracle}, we conducted an extensive series of ablation experiments. 
We first investigate the impact of removing the RDG and the effect of reducing the number of attention heads $H$ in Eq.~\eqref{eq:head} from eight to one. The results are detailed in Table~\ref{tab:ablation}. Clearly, eliminating the RDG or the multi‑head attention mechanism causes a marked decline in all evaluation metrics, highlighting their indispensibility to \textsc{GraphOracle}'s performance.

In addition, we quantified how the breadth of pre‑training data affects zero‑shot performance. Specifically, we pre‑trained on one to six heterogeneous datasets and evaluated the resulting checkpoints on unseen Table~\ref{tab:ablation}'s graphs. The averaged results, depicted in Fig.~\ref{fig:ablation_figure} (implementation details are reported in Appendix E), reveal that zero-shot performance saturates once three diverse datasets are included in the pre-training mixture. Incorporating additional datasets beyond this point yields no further significant gains. We conjecture that after this point the model has already encountered a sufficiently rich spectrum of relational patterns, and subsequent datasets may introduce largely redundant or potentially noisy signals.

Furthermore, to underscore the unique contributions of our proposed mechanisms, we compared \textsc{GraphOracle}'s relation graph construction and message passing techniques against those employed by INGRAM and ULTRA. For this, we created variants where $\text{Graph}_{\text{INGRAM}}$ denotes using INGRAM's method for relation graph construction,and $\text{Message}_{\text{INGRAM}}$ signifies adopting INGRAM's message passing scheme (similarly for ULTRA). As detailed in Table~\ref{tab:ablation}, substituting either \textsc{GraphOracle}'s graph construction or its message passing method with those from INGRAM or ULTRA resulted in a substantial reduction in performance. This comparative analysis further substantiates the effectiveness and integral role of each distinct component within the \textsc{GraphOracle} framework.

\section{Conclusion}

In this work, we introduced \textbf{\textsc{GraphOracle}}, a relation-centric foundation model for unifying reasoning across heterogeneous KGs. By converting KGs into RDG, our approach explicitly encodes compositional patterns among relations, yielding domain-invariant embeddings. Experiments on 60 diverse benchmarks showed consistent state-of-the-art performance, improving mean reciprocal rank by up to 16.8\% over baselines with minimal adaptation. We also demonstrated that \textsc{GraphOracle}'s performance can be enhanced by integrating external information through \textsc{GraphOracle+}, which leverages foundation model embeddings for improved initialization. Ablation studies confirmed the essential contributions of the relation-dependency graph and multi-head attention components. These findings establish relation-dependency pre-training as a scalable approach toward universal KG reasoning, opening new avenues for cross-domain applications.

\bibliography{aaai2026}

\clearpage

\appendix

\section*{A  Reproducibility and Code} \label{Reproducibility and Code}
To ensure reproducibility, we provide the complete source code of the \textsc{GraphOracle} framework in the supplementary materials.

\section*{B  Limitations and Future Work} \label{Limitations and Future Work}

Despite \textsc{GraphOracle}'s strong performance across various KG reasoning tasks, several limitations merit acknowledgment. The computational complexity of relation-dependency graph construction scales with the number of relations, which may present challenges for KGs with extremely high relation cardinality, potentially degrading efficiency for graphs with millions of distinct relations. Additionally, our approach currently focuses primarily on the topological structure of relation interactions and may not fully leverage all semantic nuances present in complex domain-specific knowledge, despite partial mitigation through \textsc{GraphOracle+}'s incorporation of external embeddings. Furthermore, while \textsc{GraphOracle} demonstrates strong zero-shot and few-shot capabilities, its performance still benefits from fine-tuning on target domains, indicating that truly universal KG reasoning remains challenging, particularly for highly specialized domains with unique relation structures.

The current work opens several promising directions for future research. Extending \textsc{GraphOracle} to incorporate multimodal knowledge sources represents a compelling direction where future architectures could jointly reason over textual descriptions, visual attributes, and graph structure to create more comprehensive knowledge representations, particularly valuable in domains like biomedicine where protein structures, medical images, and text reports contain complementary information. Additionally, developing temporal extensions to \textsc{GraphOracle} that model relation-dependency dynamics and knowledge evolution patterns would enable reasoning about causality, trends, and temporal dependencies between relations, addressing the static nature of current KGs. As KGs grow to include millions of relations, future research could explore techniques for automatically discovering relation taxonomies and leveraging them to create more efficient and scalable message-passing architectures through hierarchical abstractions of relation-dependencies.

\section*{C Time Complexity Analysis}  \label{Complexity Analysis}
\textsc{GraphOracle}’s overall time complexity consists of two parts: a one‑time preprocessing cost and a per‑query inference cost. Building the Relation–Dependency Graph (RDG) by scanning all triples once requires
$
  \mathcal{O}(|F|),
$
where \(|F|\) is the total number of triples. For each query \((e_q, r_q, ?)\), the $L_R$ layers of relation‑level message passing on the RDG incur
$
  \mathcal{O}(L_R \, |E_R| \, d),
$
where \(|E_R|\) is the number of edges in the RDG and \(d\) is the hidden dimension; in practice \(|E_R|\ll|F|\) and \(L_R\le3\), so this cost is small. Subsequently, the $L_E$ layers of entity‑level message passing propagate representations across an average branching factor \(b\), costing
$
  \mathcal{O}(L_E \, b \, d).
$
With typical settings \(L_E\le3\) and \(b<30\), this yields sub‑millisecond latency per query. Hence the end‑to‑end per‑query complexity is
$
  \mathcal{O}\bigl(L_R\,|E_R|\,d + L_E\,b\,d\bigr),
$
while preprocessing remains \(\mathcal{O}(|F|)\). Thanks to small constant depths and modest branching, \textsc{GraphOracle} achieves near‑linear scalability and memory‑efficient inference on large, heterogeneous knowledge graphs.

\section*{D  Statistics of Datasets} \label{Statistics of Datasets}

We choose the same datasets as ULTRA~\cite{galkin2024ultra},
TRIX~\cite{zhang2025trix} and
KG-ICL~\cite{cui2024promptkg}. Details of the used Knowledge Graph datasets are given in Table~\ref{Table:datasets statistics} and Table~\ref{Table:datasets statistics2}.
Furthermore, we first create three cross-domain datasets for cross-domain Knowledge Graph Reasoning, whose details can be found in Appendix F.

\begin{table*}[htbptpb]
\caption{Statistics of the KG datasets. $Q_{\text{tra}}$, $Q_{\text{val}}$, $Q_{\text{tst}}$ are the query triplets used for reasoning.}
\resizebox{\linewidth}{!}{
\begin{tabular}{ccccccccc}
\hline
Dataset & Reference &  \# Entity & \# Relation & $|\mathcal E|$ & $|Q_{\text{tra}}|$ & $|Q_{\text{val}}|$ & $|Q_{\text{tst}}|$ & Supervised SOTA\\
\hline
WN18RR & Dettmers
et al. (2017)~\cite{dettmers2017convolutional}&40.9k & 11  & 65.1k & 21.7k & 3.0k & 3.1k &one-shot-subgraph \cite{one-shot-subgraph} \\
FB15k237 &Toutanova and Chen (2015)~\cite{toutanova2015observed} & 14.5k & 237 & 204.1k & 68.0k & 17.5k & 20.4k  & Adaprop \cite{Adaprop}\\
NELL-995 & Xiong et al.(2017)~\cite{xiong2017deeppath}&  74.5k & 200  & 112.2k & 37.4k & 543 & 2.8k & Adaprop \cite{Adaprop}\\
YAGO3-10 & Suchanek et al.(2007)~\cite{suchanek2007yago}&123.1k & 37  & 809.2k & 269.7k & 5.0k & 5.0k &one-shot-subgraph \cite{one-shot-subgraph} \\
\hline
Nell-V1 & Teru et al. (2020)~\cite{TER2020}& 3.1k & 14 & 5.5k & 4.7k & 0.4k & 0.4k & KG-ICL~\cite{cui2024promptkg} \\
Nell-V2 & Teru et al. (2020)~\cite{TER2020} & 2.6k & 88 & 10.1k & 8.2k & 90.9k & 1.0k& KG-ICL~\cite{cui2024promptkg} \\
Nell-V3 & Teru et al. (2020)~\cite{TER2020} & 4.6k & 142 & 20.1k & 16.4k & 1.9k & 1.9k & KG-ICL~\cite{cui2024promptkg} \\
Nell-V4 & Teru et al. (2020)~\cite{TER2020} & 2.1k & 76 & 9.3k & 7.5k & 0.9k & 0.9k & KG-ICL~\cite{cui2024promptkg} \\
WN-V1 & Teru et al. (2020)~\cite{TER2020}& 2.7k & 9 & 6.7k & 5.4k & 0.6k & 0.6k & KG-ICL~\cite{cui2024promptkg} \\
WN-V2 & Teru et al. (2020)~\cite{TER2020} & 7.0k & 10 & 20.0k & 15.3k & 1.8k & 1.9k& KG-ICL~\cite{cui2024promptkg} \\
WN-V3 & Teru et al. (2020)~\cite{TER2020} & 12.1k & 11 & 32.2k & 25.9k & 3.1k & 3.2k & KG-ICL~\cite{cui2024promptkg} \\
WN-V4 & Teru et al. (2020)~\cite{TER2020} & 3.9k & 9 & 9.8k & 7.9k & 0.9k & 1.0k & KG-ICL~\cite{cui2024promptkg} \\
FB-V1 & Teru et al. (2020)~\cite{TER2020}& 1.6k & 180 & 5.3k & 4.2k & 0.5k & 0.5k & KG-ICL~\cite{cui2024promptkg} \\
FB-V2 & Teru et al. (2020)~\cite{TER2020} & 2.6k & 200 & 12.1k & 9.7k & 1.2k & 1.2k& KG-ICL~\cite{cui2024promptkg} \\
FB-V3 & Teru et al. (2020)~\cite{TER2020} & 3.7k & 215 & 22.4k & 18.0k & 2.2k & 2.2k & KG-ICL~\cite{cui2024promptkg} \\
FB-V4 & Teru et al. (2020)~\cite{TER2020} & 4.7k & 219 & 33.9k & 27.2k & 3.4k & 3.4k & KG-ICL~\cite{cui2024promptkg} \\
\hline
Nell-25 & Lee et al. (2023)~\cite{lee2023ingram} & 5.2k & 146 & 19.1k & 17.6k & 0.7k & 0.7k & KG-ICL~\cite{cui2024promptkg} \\
Nell-50 & Lee et al. (2023)~\cite{lee2023ingram} & 5.3k & 150 & 19.3k & 17.6k & 0.9k & 0.9k& KG-ICL~\cite{cui2024promptkg} \\
Nell-75 &  Lee et al. (2023)~\cite{lee2023ingram} & 3.3k & 138 & 12.3k & 11.1k & 6.1k & 6.1k & KG-ICL~\cite{cui2024promptkg} \\
Nell-100 &  Lee et al. (2023)~\cite{lee2023ingram} & 2.1k & 99 & 9.4k & 7.8k & 0.8k & 0.8k & KG-ICL~\cite{cui2024promptkg} \\
WK-25 & Lee et al. (2023)~\cite{lee2023ingram} & 13.9k & 67 & 44.1k & 41.9k & 1.1k & 1.1k & KG-ICL~\cite{cui2024promptkg} \\
WK-50 & Lee et al. (2023)~\cite{lee2023ingram} & 16.3k & 102 & 88.9k & 82.5k & 3.2k & 3.2k& KG-ICL~\cite{cui2024promptkg} \\
WK-75 &  Lee et al. (2023)~\cite{lee2023ingram} & 8.1k & 77 & 31.0k & 28.7k & 1.1k & 1.1k & KG-ICL~\cite{cui2024promptkg} \\
WK-100 &  Lee et al. (2023)~\cite{lee2023ingram} & 15.9k & 103 & 58.9k & 49.9k & 4.5k & 4.5k & KG-ICL~\cite{cui2024promptkg} \\
FB-25 & Lee et al. (2023)~\cite{lee2023ingram} & 8.7k & 233 & 103.0k & 91.6k & 5.7k & 5.7k & KG-ICL~\cite{cui2024promptkg} \\
FB-50 & Lee et al. (2023)~\cite{lee2023ingram} & 8.6k & 228 & 93.1k & 85.4k & 3.9k & 3.9k& KG-ICL~\cite{cui2024promptkg} \\
FB-75 &  Lee et al. (2023)~\cite{lee2023ingram} & 6.9k & 213 & 69.0k & 62.8k & 3.1k & 3.1k & KG-ICL~\cite{cui2024promptkg} \\
FB-100 &  Lee et al. (2023)~\cite{lee2023ingram} & 6.5k & 202 & 67.5k & 62.8k & 2.3k & 2.3k & KG-ICL~\cite{cui2024promptkg}\\
\hline

PrimeKG & Chandak et al. (2023)~\cite{chandak2023building} & 85.0k & 14 & 3911.9k & 3129.8k & 391.2k & 391.2k & Adaprop~\cite{Adaprop} \\
Amazon-book & Wang et al. (2019)~\cite{wang2019kgat} & 3404.2k & 40 & 3404.2k & 3210.3k  & 98.0k &  95.9k & KUCNet~\cite{liu2024kucnet} \\
GeoKG & GeoNames Team (2025)~\cite{geonames} & 2054.2k & 681 & 2784.5k & 2673.1k & 55.7k & 55.7k & one-shot-subgraph \cite{one-shot-subgraph} \\

\hline
\end{tabular}

\label{Table:datasets statistics}
}
\end{table*}

\begin{table*}[htbptpb]
\caption{Statistics of the KG datasets. $Q_{\text{tra}}$, $Q_{\text{val}}$, $Q_{\text{tst}}$ are the query triplets used for reasoning.}
\resizebox{\linewidth}{!}{
\begin{tabular}{ccccccccc}
\hline
Dataset & Reference &  \# Entity & \# Relation & $|Q_{\text{tra}}|$ & $|Q_{\text{val}}|$ & $|Q_{\text{tst}}|$ & Supervised SOTA\\
\hline
CoDEx Small & Dettmers et al. (2020) & 2.0k & 42  & 32.9k & 1.8k & 1.8k  &ULTRA~\cite{galkin2024ultra} \\
CoDEx Medium & Dettmers et al. (2020) & 17.1k & 51  & 185.6k & 10.3k & 10.3k  &KG-ICL~\cite{cui2024promptkg} \\
CoDEx Large & Dettmers et al. (2020) & 78.0k & 69  & 551.2k & 30.6k & 30.6k  &KG-ICL~\cite{cui2024promptkg} \\
WDsinger & Lv et al. (2020) & 10.3k & 135  & 16.1k & 2.2k & 2.2k  &TRIX~\cite{zhang2025trix} \\
NELL23k & Lv et al. (2020) & 22.9k & 200  & 25.4k & 5.0k & 5.0k  &KG-ICL~\cite{cui2024promptkg} \\
FB15k237\_10 & Lv et al. (2020) & 11.5k & 237  & 27.2k & 15.6k & 18.2k  &KG-ICL~\cite{cui2024promptkg} \\
FB15k237\_20 & Lv et al. (2020) & 13.2k & 237  & 54.4k & 17.0k & 20.0k  &KG-ICL~\cite{cui2024promptkg} \\
FB15k237\_50 & Lv et al. (2020) & 14.1k & 237  & 136.1k & 17.4k & 20.3k  &ULTRA~\cite{galkin2024ultra} \\
DBpedia100k & Ding et al. (2018) & 99.6k & 470  & 597.6k & 50.0k & 50.0k  &TRIX~\cite{zhang2025trix} \\
AristoV4 & Chen et al. (2021) & 45.0k & 1605  & 242.6k & 20.0k & 20.0k  &TRIX~\cite{zhang2025trix} \\
ConceptNet100k & Malaviya et al. (2020) & 78.3k & 34  & 100.0k & 1.2k & 1.2k  &KG-ICL~\cite{cui2024promptkg} \\
Hetionet & Himmelstein et al. (2017) & 45.2k & 24  & 2025.2k & 112.5k & 112.5k  &ULTRA~\cite{galkin2024ultra} 
\\
\hline

\end{tabular}

\label{Table:datasets statistics2}
}
\end{table*}

\section*{E  Comprehensive Training Details}\label{Training details}

\paragraph{Sequential multi‑dataset schedule.}  
Given $K$ KGs $\{\mathcal{G}_1,\ldots,\mathcal{G}_K\}$ sorted by domain diversity, we train \textsc{GraphOracle} sequentially from $\mathcal{G}_1$ to $\mathcal{G}_K$. Parameters are \emph{rolled over} between datasets to accumulate relational knowledge.

\paragraph{Dataset‑specific hyper‑parameters.}  
Each dataset $\mathcal{G}_k$ employs a tuple  
$\Theta_k=\{\alpha_k,\lambda_k,\gamma_k,d_k^{(\mathrm{h})},d_k^{(\mathrm{a})},\delta_k,\mathcal{A}_k,L_k\}$  
denoting learning rate, $\ell_2$ regularization, decay factor, hidden dimension, attention dimension, dropout rate, activation function, and layer count, respectively. Values are chosen via grid search on the validation split of $\mathcal{G}_k$.

\paragraph{Objective.}  
\begin{align}
\mathcal{L}^{(k)}
      &=\mathcal{L}_{\mathrm{train}}^{(k)}
        +\lambda_k\lVert\Theta\rVert_2^2, \\
\mathcal{L}_{\mathrm{train}}^{(k)}
      &=\mathbb{E}_{(h,r,t)\sim\mathcal{D}_k}
        \bigl[-\log p(t\mid h,r;\Theta)\bigr],
\end{align}
with negative sampling ratio $N_{\text{neg}}=64$.

\paragraph{Learning‑rate decay and early stopping.}  
At epoch $\epsilon$ we apply  
$\alpha_k^{(t+1)}=\gamma_k\!\cdot\!\alpha_k^{(t)}$; training terminates when validation MRR has not improved for $10$ epochs.

\paragraph{Parameter transfer.}  
After convergence on $\mathcal{G}_k$, we initialize the next run via
\begin{equation}
\Theta_{k+1}^{(0)}
      =\mathcal{T}(\Theta_k^\star),
\end{equation}
where $\mathcal{T}$ preserves (i) relation‑dependency graph weights and (ii) shared layer norms, while re‑initializing dataset‑specific embeddings.

\begin{table}[h]
\centering
\caption{Graphs in different pre-training mixtures.}
\small
\begin{tabular}{lccccccc}
\toprule
 Dataset& 1 & 2 & 3 & 4 & 5 & 6\\
\midrule
WN18RR          & \cmark & \cmark & \cmark & \cmark & \cmark & \cmark  \\
CoDEx-Medium     &        & \cmark & \cmark & \cmark & \cmark & \cmark \\
 FB15k237         &        &        & \cmark & \cmark & \cmark & \cmark \\
NELL995          &        &        &        & \cmark & \cmark & \cmark  \\
PrimeKG         &        &        &        & & \cmark & \cmark  \\
GeoKG      &        &        &        &  &  & \cmark \\

\midrule
Batch size       & 50     & 10     & 10     & 5     & 2     & 2     \\
\bottomrule
\end{tabular}
\label{Tab:mix_of_datasets}
\end{table}

\begin{table}[htbptbp]
\centering
\caption{Hyperparameters used across different datasets.}
\resizebox{0.5\textwidth}{!}{
\begin{tabular}{lcccccc}
\hline
\textbf{Datasets} & \textbf{Learning Rate} & \textbf{Act} & \textbf{Entity Layer} & \textbf{Relation Layer} & \textbf{Hidden Dim} & \textbf{Batch Size} \\
\hline
WN18RR      & 0.003  & idd  & 5 & 3 & 64 & 50 \\
FB15k237    & 0.0009 & relu & 4 & 4 & 48 & 10 \\
NELL-995    & 0.0011 & relu & 5 & 4 & 48 &  5 \\
YAGO3-10    & 0.001  & relu & 7 & 4 & 64 &  5 \\
\hline
NELL-100    & 0.0016 & relu & 5 & 3 & 48 & 10 \\
NELL-75     & 0.0013 & relu & 5 & 3 & 48 & 10 \\
NELL-50     & 0.0015 & tanh & 5 & 3 & 48 & 10 \\
NELL-25     & 0.0016 & relu & 5 & 3 & 48 & 10 \\
WK-100      & 0.0027 & relu & 5 & 3 & 48 & 10 \\
WK-75       & 0.0018 & relu & 5 & 3 & 48 & 10 \\
WK-50       & 0.0022 & relu & 5 & 3 & 48 & 10 \\
WK-25       & 0.0023 & idd  & 5 & 3 & 48 & 10 \\
FB-100      & 0.0043 & relu & 5 & 3 & 48 & 10 \\
FB-75       & 0.0037 & relu & 5 & 3 & 48 & 10 \\
FB-50       & 0.0008 & relu & 5 & 3 & 48 & 10 \\
FB-25       & 0.0005 & tanh & 5 & 3 & 16 & 24 \\
\hline
WN-V1       & 0.005  & idd  & 5 & 3 & 64 & 100 \\
WN-V2       & 0.0016 & relu & 5 & 4 & 48 & 20 \\
WN-V3       & 0.0014 & tanh & 5 & 4 & 64 & 20 \\
WN-V4       & 0.006  & relu & 5 & 3 & 32 & 10 \\
FB-V1       & 0.0092 & relu & 5 & 3 & 32 & 20 \\
FB-V2       & 0.0077 & relu & 3 & 3 & 48 & 10 \\
FB-V3       & 0.0006 & relu & 3 & 3 & 48 & 20 \\
FB-V4       & 0.0052 & idd  & 5 & 4 & 48 & 20 \\
NL-V1       & 0.0021 & relu & 5 & 3 & 48 & 10 \\
NL-V2       & 0.0075 & relu & 3 & 3 & 48 & 100 \\
NL-V3       & 0.0008 & relu & 3 & 3 & 16 & 10 \\
NL-V4       & 0.0005 & tanh & 5 & 4 & 16 & 20 \\
\hline
PrimeKG     & 0.00016& relu & 5 & 4 & 16 &  2 \\
Amazon-book & 0.0002 & idd  & 5 & 3 & 32 &  3 \\
GeoKG       & 0.0005 & relu & 4 & 3 & 16 &  2 \\
\hline
\end{tabular}
}
\label{tab:hyperparameters}
\end{table}

We perform a grid search and use the Optuna library to search for the optimal hyperparameters. Table~\ref{tab:hyperparameters} presents the choices of hyperparameters on all datasets, and Table~\ref{Tab:mix_of_datasets} presents the choices of datasets in the pre-training process.

\section*{F  Cross-domain Dataset Processing Details} \label{Training Detail of Cross domain Dataset}

The datasets used in this paper are all open source and can be obtained from:
\begin{itemize}[leftmargin=*]
    \item The biomedical domain datasets are publicly available at \url{https://dataverse.harvard.edu/dataset.xhtml?persistentId=doi:10.7910/DVN/IXA7BM}.
    \item The recommendation domain Amazon-book is avaliable at \url{http://jmcauley.ucsd.edu/data/amazon}.
    \item The geographic datasets (GeoKG) are avaliable at \url{https://www.geonames.org/}.
\end{itemize}

\subsection*{F.1 Processing Details of Biomedical Domain}

\begin{table*}[htbpt]
\centering
\caption{Comparison on PrimeKG. Best performance is highlighted with \textbf{bold}, and the second best is \underline{underlined}. \textsc{GraphOracle}-S means the \textsc{GraphOracle} was trained from scratch, while \textsc{GraphOracle}-F means the \textsc{GraphOracle} was trained by fine-tuning.}
\resizebox{\textwidth}{!}{
\begin{tabular}{lc|ccc|ccc|ccc|ccc|ccc|ccc|ccc}
\hline
\multirow{2}{*}{Type} & \multirow{2}{*}{Method} 
& \multicolumn{3}{c|}{Protein$\rightarrow$BP} 
& \multicolumn{3}{c|}{Protein$\rightarrow$MF} 
& \multicolumn{3}{c|}{Protein$\rightarrow$CC} 
& \multicolumn{3}{c|}{Drug$\rightarrow$Disease} 
& \multicolumn{3}{c|}{Protein$\rightarrow$Drug} 
& \multicolumn{3}{c|}{Disease$\rightarrow$Protein} 
& \multicolumn{3}{c}{Drug$\not\!\rightarrow$Disease} \\ 
 &  & MRR & H@1 & H@10 & MRR & H@1 & H@10 & MRR & H@1 & H@10 & MRR & H@1 & H@10
 & MRR & H@1 & H@10 & MRR & H@1 & H@10 & MRR & H@1 & H@10 \\
\hline
\multirow{8}{*}{Embedding} & TransE & 0.034 & 0.023 & 0.052 & 0.046 & 0.034 & 0.073 & 0.044 & 0.027 & 0.074 & 0.017 & 0.010 & 0.030 & 0.033 & 0.022 & 0.053 & 0.024 & 0.014 & 0.045 & 0.010 & 0.005 & 0.019 \\
 & TransR & 0.045 & 0.030 & 0.068 & 0.060 & 0.044 & 0.095 & 0.048 & 0.030 & 0.081 & 0.053 & 0.032 & 0.093 & 0.069 & 0.046 & 0.112 & 0.028 & 0.016 & 0.052 & 0.029 & 0.015 & 0.055 \\
 & TransSH & 0.044 & 0.029 & 0.067 & 0.061 & 0.045 & 0.096 & 0.057 & 0.035 & 0.096 & 0.026 & 0.016 & 0.046 & 0.043 & 0.028 & 0.070 & 0.024 & 0.014 & 0.045 & 0.014 & 0.007 & 0.026 \\
 & TransD & 0.043 & 0.029 & 0.065 & 0.059 & 0.044 & 0.093 & 0.053 & 0.033 & 0.090 & 0.022 & 0.013 & 0.039 & 0.049 & 0.032 & 0.079 & 0.024 & 0.014 & 0.045 & 0.013 & 0.007 & 0.025 \\
 & ComplEx & 0.084 & 0.056 & 0.128 & 0.100 & 0.074 & 0.158 & 0.099 & 0.061 & 0.167 & 0.042 & 0.025 & 0.074 & 0.079 & 0.052 & 0.128 & 0.059 & 0.034 & 0.110 & 0.048 & 0.025 & 0.091 \\
 & DistMult & 0.054 & 0.036 & 0.082 & 0.089 & 0.066 & 0.141 & 0.095 & 0.059 & 0.161 & 0.025 & 0.015 & 0.044 & 0.044 & 0.029 & 0.071 & 0.033 & 0.019 & 0.062 & 0.047 & 0.025 & 0.089 \\
 & RotatE & 0.079 & 0.053 & 0.120 & 0.119 & 0.088 & 0.188 & 0.107 & 0.066 & 0.181 & 0.150 & 0.090 & 0.264 & 0.125 & 0.083 & 0.203 & 0.070 & 0.041 & 0.131 & 0.076 & 0.040 & 0.144 \\
 & BioBRIDGE & 0.136 & 0.091 & 0.207 & 0.326 & 0.241 & 0.515 & 0.319 & 0.198 & 0.539 & 0.189 & 0.113 & 0.333 & 0.172 & 0.114 & 0.279 & 0.084 & 0.049 & 0.157 & 0.081 & 0.043 & 0.153 \\
\hline
\multirow{9}{*}{GNNs} & NBFNet & 0.279 & 0.187 & 0.424 & 0.335 & 0.248 & 0.529 & 0.321 & 0.199 & 0.543 & 0.169 & 0.101 & 0.297 & 0.156 & 0.103 & 0.253 & 0.200 & 0.116 & 0.374 & 0.139 & 0.074 & 0.263 \\
 & RED-GNN & 0.284 & 0.190 & 0.432 & 0.341 & 0.252 & 0.539 & 0.327 & 0.203 & 0.553 & 0.172 & 0.103 & 0.303 & 0.159 & 0.105 & 0.258 & 0.203 & 0.118 & 0.379 & 0.142 & 0.075 & 0.268 \\
 & A*Net & 0.317 & 0.212 & 0.482 & 0.381 & 0.282 & 0.602 & 0.365 & 0.226 & 0.617 & 0.192 & 0.115 & 0.338 & 0.177 & 0.117 & 0.287 & 0.227 & 0.132 & 0.424 & 0.158 & 0.084 & 0.299 \\
 & AdaProp & 0.334 & 0.224 & \underline{0.508} & 0.402 & 0.297 & 0.635 & 0.385 & 0.239 & \underline{0.651} & \underline{0.202} & \underline{0.121} & \underline{0.356} & 0.187 & 0.124 & 0.303 & 0.239 & 0.139 & 0.447 & 0.167 & 0.089 & 0.316 \\
 & one-shot-subgraph & 0.231 & 0.155 & 0.351 & 0.278 & 0.206 & 0.439 & 0.266 & 0.165 & 0.450 & 0.140 & 0.084 & 0.246 & 0.129 & 0.085 & 0.209 & 0.165 & 0.096 & 0.308 & 0.115 & 0.061 & 0.217 \\
 & INGRAM & 0.269 & 0.180 & 0.409 & 0.324 & 0.240 & 0.512 & 0.310 & 0.192 & 0.524 & 0.163 & 0.098 & 0.287 & 0.151 & 0.100 & 0.245 & 0.193 & 0.112 & 0.361 & 0.134 & 0.071 & 0.253 \\
 & ULTRA & 0.313 & 0.210 & 0.476 & 0.376 & 0.278 & 0.594 & 0.360 & 0.223 & 0.608 & 0.189 & 0.113 & 0.333 & 0.175 & 0.116 & 0.284 & 0.224 & 0.130 & 0.419 & 0.156 & 0.083 & 0.295 \\
 \cline{2-23}
 & \textsc{GraphOracle}-S & \underline{0.392} & \underline{0.251} & 0.500 & \underline{0.423} & \underline{0.314} & \underline{0.698} & \underline{0.431} & \underline{0.263} & 0.692 & \underline{0.223} & \underline{0.132} & \underline{0.397} & \underline{0.197} & \underline{0.139} & \underline{0.334} & \underline{0.262} & \underline{0.154} & \underline{0.487} & \underline{0.176} & \underline{0.103} & \underline{0.345} \\
 &\textsc{GraphOracle}-F & \textbf{0.498} & \textbf{0.323} & \textbf{0.684} & \textbf{0.499} & \textbf{0.366} & \textbf{0.748} & \textbf{0.475} & \textbf{0.325} & \textbf{0.752} & \textbf{0.268} & \textbf{0.175} & \textbf{0.448} & \textbf{0.232} & \textbf{0.167} & \textbf{0.378} & \textbf{0.299} & \textbf{0.187} & \textbf{0.531} & \textbf{0.192} & \textbf{0.135} & \textbf{0.380} \\
\hline
\end{tabular}
}
\label{tab:primekg}
\end{table*}

\subsubsection{Dataset Creation}
\begin{table*}[t]
\caption{Statistical analysis of node and edge distribution in the original PrimeKG dataset and our processed KG used for training. ``Original'' represents the raw PrimeKG data; ``Processed'' indicates our filtered KG; ``Dropped'' shows the number of entities removed during preprocessing.}
\label{tab:kg_statistics}
\centering
\small
\begin{tabular}{llrrrr}
\toprule
\textbf{Type} & \textbf{Modality} & \textbf{Original} & \textbf{Processed} & \textbf{Dropped} & \textbf{Percent dropped} \\
\midrule
\multirow{7}{*}{Nodes} & biological process & 28,642 & 27,478 & 1,164 & 4.06\% \\ 
 & protein & 27,671 & 19,162 & 8,509 & 30.75\% \\
 & disease & 17,080 & 17,080 & 0 & 0.00\% \\
 & molecular function & 11,169 & 10,966 & 203 & 1.82\% \\
 & drug & 7,957 & 6,948 & 1,009 & 12.68\% \\
 & cellular component & 4,176 & 4,013 & 163 & 3.90\% \\
 & Summation & 96,695 & 85,647 & 11,048 & 11.43\% \\
\midrule
\textbf{Type} & \textbf{Relation} & \textbf{Original} & \textbf{Processed} & \textbf{Dropped} & \textbf{Percent Dropped} \\
\midrule
\multirow{15}{*}{Edges} & drug\_drug & 2,672,628 & 2,241,466 & 431,162 & 16.13\% \\ 
 & protein\_protein & 642,150 & 629,208 & 12,942 & 2.02\% \\
 & bioprocess\_protein & 289,610 & 272,642 & 16,968 & 5.86\% \\
 & cellcomp\_protein & 166,804 & 149,504 & 17,300 & 10.37\% \\
 & disease\_protein & 160,822 & 155,924 & 4,898 & 3.05\% \\
 & molfunc\_protein & 139,060 & 133,522 & 5,538 & 3.98\% \\
 & bioprocess\_bioprocess & 105,772 & 99,630 & 6,142 & 5.81\% \\
 & disease\_disease & 64,388 & 64,388 & 0 & 0.00\% \\
 & contraindication & 61,350 & 60,130 & 1,220 & 1.99\% \\
 & drug\_protein & 51,306 & 47,614 & 3,692 & 7.20\% \\
 & molfunc\_molfunc & 27,148 & 26,436 & 712 & 2.62\% \\
 & indication & 18,776 & 17,578 & 1,198 & 6.38\% \\
 & cellcomp\_cellcomp & 9,690 & 9,200 & 490 & 5.06\% \\
 & off-label use & 5,136 & 4,998 & 138 & 2.69\% \\
 & Summation & 4,414,640 & 3,912,240 & 502,400 & 11.38\% \\
\bottomrule
\label{tab:detail_of_BioKG}
\end{tabular}
\end{table*}

Our research establishes a framework for integrating uni-modal foundation models through KG simplification. As show in Table~\ref{tab:detail_of_BioKG}, for experimental efficiency, we refined the KG to include six key modalities. These retained modalities—protein, disease, drug, and gene ontology terms—represent the core biomedical entities crucial for addressing real-world applications including drug discovery, repurposing, protein-protein interaction analysis, protein function prediction, and drug-target interaction modeling. Table~\ref{tab:kg_statistics} presents a comprehensive comparison between the original and our processed KG.

For classic KG reasoning dataset, we retain only the IDs of entities and relations, while preserving their modalities as auxiliary information.   These modalities are used solely for categorizing relation types during statistical analysis and are not involved in the message passing process. The dataset is partitioned into training, validation, and test sets with a standard split of 80\%, 10\%, and 10\%, respectively.

\subsubsection{External Information-Enrichment Dataset Creation}

The initial PrimeKG dataset aggregates biomedical entities from numerous sources. To enhance the utility of this dataset for our purposes, we enriched entities with essential properties and established connections to external knowledge bases, removing entities lacking required attributes. 
\paragraph{Protein Entities}
The original PrimeKG contains 27,671 protein entries. We implemented a mapping procedure to associate these proteins with UniProtKB/Swiss-Prot sequence database through the UniProt ID mapping service (\url{https://www.uniprot.org/id-mapping}). This procedure yielded 27,478 protein sequences successfully matched with gene identifiers.

Further analysis of the unmapped entries revealed that most corresponded to non-protein-coding genetic elements (including pseudogenes, rRNA, and ncRNA genes), which do not produce functional proteins. Given our focus on protein-centric applications, excluding these entries was appropriate.

\paragraph{Drug Entities}
From the initial 7,957 drug entries in PrimeKG, we performed identity matching against the DrugBank database (\url{https://go.drugbank.com/drugs}). During this process, we removed drugs lacking SMILES structural notation, resulting in 6,948 validated drug entities for our training dataset.

\paragraph{Gene Ontology Terms}
The biological process, molecular function, and cellular component categories comprise the Gene Ontology (GO) terminology in our dataset. We utilized AmiGO to extract detailed definitions of these GO terms through their identifiers (\url{https://amigo.geneontology.org/amigo/search/ontology}). This process allowed us to incorporate 27,478 biological process terms, 10,966 molecular function terms, and 4,013 cellular component terms into our training dataset.

\paragraph{Disease Entities}
Disease descriptions were directly adopted from the PrimeKG dataset, allowing us to retain all 17,080 disease entities without modification for training purposes.

\subsection*{F.2 Processing Detail of Recommendation Domain}

\begin{table*}[htbpt]
\centering
\caption{Comparison of the performance of different methods on Amazon-book. The best performance is marked in \textbf{bold} and the second best performance is \underline{underlined}. The \textsc{GraphOracle}-S means the \textsc{GraphOracle} was train from scratch while \textsc{GraphOracle}-F means the \textsc{GraphOracle} was trained by finetune.}
\resizebox{\textwidth}{!}{
\begin{tabular}{c|cccccccccccccccc}
\hline

Method  & MF & FM & NFM & RippleNet & KGNN-LS & CKAN & KGIN & CKE & R-GCN & KGAT & PPR & PathSim & RED-GNN & KUCNet & \textsc{GraphOracle}-S & \textsc{GraphOracle}-F  \\
\hline
 Recall@20 & 0      & 0.0026 & 0.0006 & 0.0011 & 0.0001 & 0.0005 & 0.0868 & 0 & 0.0001 & 0.0001 & 0.0301 & 0.2053 & 0.2187 & 0.2237 & \underline{0.2453} & \textbf{0.3142}   \\
 NDCG@20 & 0     & 0.0010 & 0.0003 & 0.0005 & 0.0001 & 0.0003 & 0.0446 & 0 & 0.0001 & 0.0001 & 0.0167 & 0.1491 & 0.1633 & 0.1685 & \underline{0.1987} & \textbf{0.2591} \\
\hline
\end{tabular}
}
\label{tab:Amazon-book}
\end{table*}

\begin{algorithm}[ht]
\SetAlgoLined
\SetKwComment{Comment}{\quad// }{}
\caption{KG Enrichment \& Entity Canonicalization}
\KwIn{$\mathcal{R}$ (relation list), $\mathcal{T}_{\text{train}}$ (training triples), $\mathcal{K}_G$ (original KG), $\mathcal{T}_{\text{test}}$ (test triples)}
\KwOut{$\mathcal{R}'$ (extended relations), $\mathcal{K}'_G$ (enriched KG), $\mathcal{M}_E$ (entity--ID map)}

\tcc*[f]{Phase I: Relation Ontology Extension}\;
$\bm{\Omega}\leftarrow\textsc{ReadRelations}(\mathcal{R})$\;
$r_{\max}\leftarrow\max_{r\in\bm{\Omega}}\textsc{id}(r)$\;
$\text{id}_{\text{purchase}}\leftarrow r_{\max}+1$\;
$\mathcal{R}'\leftarrow\bm{\Omega}\cup\{(\text{``purchase''},\text{id}_{\text{purchase}})\}$\;
\textsc{PersistRelationSet}$(\mathcal{R}')$\;

\tcc*[f]{Phase II: Semantic Triple Generation}\;
$\mathcal{P}\leftarrow\emptyset$\;
\ForEach{$\sigma\in\mathcal{T}_{\text{train}}$}{
    $\mathcal{E}\leftarrow\textsc{TokenizeEntities}(\sigma)$\;
    $u\leftarrow\mathcal{E}[0]$\Comment*[r]{user}
    $\mathcal{I}_u\leftarrow\mathcal{E}[1:]$\Comment*[r]{items}
    \ForEach{$i\in\mathcal{I}_u$}{
        $\mathcal{P}\leftarrow\mathcal{P}\cup\{(u,\text{id}_{\text{purchase}},i)\}$\;
    }
}

\tcc*[f]{Phase III: KG Augmentation}\;
$\mathcal{K}_{\text{ori}}\leftarrow\textsc{ExtractTriples}(\mathcal{K}_G)$\;
$\mathcal{K}'_G\leftarrow\mathcal{K}_{\text{ori}}\cup\mathcal{P}$\;
\textsc{PersistEnrichedKG}$(\mathcal{K}'_G)$\;

\tcc*[f]{Phase IV: Entity Canonicalization}\;
$\mathcal{E}_{\text{train}}\leftarrow\{s,o\mid(s,\_,o)\in\mathcal{K}'_G\}$\;
$\mathcal{E}_{\text{test}}\leftarrow\bigcup_{\sigma\in\mathcal{T}_{\text{test}}}\textsc{TokenizeEntities}(\sigma)$\;
$\mathcal{E}_{\text{uni}}\leftarrow\mathcal{E}_{\text{train}}\cup\mathcal{E}_{\text{test}}$\;
$\mathcal{E}_{\text{ord}}\leftarrow\textsc{TopologicalSort}(\mathcal{E}_{\text{uni}})$\;
\For{$j\leftarrow0$ \KwTo $|\mathcal{E}_{\text{ord}}|-1$}{
    $\Phi(\mathcal{E}_{\text{ord}}[j])\leftarrow j$\;
}
$\mathcal{M}_E\leftarrow\{(e,\Phi(e))\mid e\in\mathcal{E}_{\text{ord}}\}$\;
\textsc{PersistEntityMapping}$(\mathcal{M}_E)$\;
\KwRet{$\mathcal{R}',\mathcal{K}'_G,\mathcal{M}_E$}
\end{algorithm}

\textbf{Theoretical Framework and Implementation Principles}

The algorithm presented herein delineates a comprehensive methodology for KG enrichment and entity canonicalization within the recommendation domain. This approach operates through a multi-phase framework that transforms heterogeneous data sources into a unified semantic representation amenable to graph-based recommendation algorithms.

Phase I (\textit{Relation Ontology Extension}) introduces a formal extension of the relational schema $\mathcal{R}$ with a domain-specific "purchase" relation. Let $\bm{\Omega} = \{(r_1, id_1), (r_2, id_2), ..., (r_n, id_n)\}$ represent the initial relation set. The algorithm derives $r_{max} = \max_{r \in \bm{\Omega}}(\text{id}(r))$ and establishes $\text{id}_{purchase} = r_{max} + 1$, thus creating an extended relation ontology $\mathcal{R}' = \bm{\Omega} \cup \{(\text{"purchase"}, \text{id}_{purchase})\}$. This expansion facilitates the semantic representation of user-item interactions within the KG structure.

Phase II (\textit{Semantic Triple Generation}) transforms implicit user-item interactions into explicit RDF-compatible triples. For each user $u \in \mathcal{U}$ and their associated items $\mathcal{I}_u \subseteq \mathcal{I}$, where $\mathcal{U}$ and $\mathcal{I}$ denote the user and item entity spaces respectively, the algorithm constructs a set of purchase triples defined by:
\begin{equation}
    \mathcal{P} = \bigcup_{u \in \mathcal{U}} \bigcup_{i \in \mathcal{I}_u} \{(u, \text{id}_{purchase}, i)\}
\end{equation}
These triples codify the user-item engagement patterns within the formalism of a KG, enabling the integration of collaborative filtering signals with semantic relationships.

Phase III (\textit{KG Augmentation}) implements the fusion of the original KG $\mathcal{K}_{original}$ with the newly derived purchase triples $\mathcal{P}$. The enriched KG $\mathcal{K}'_G$ is formalized as:
\begin{equation}
    \mathcal{K}'_G = \mathcal{K}_{original} \cup \mathcal{P}
\end{equation}
This augmentation creates a multi-relational graph structure that encapsulates both semantic domain knowledge and behavioral interaction patterns.

Phase IV (\textit{Entity Canonicalization}) establishes a unified reference framework for all entities across both training and evaluation datasets. The algorithm constructs the universal entity set $\mathcal{E}_{universal} = \mathcal{E}_{train} \cup \mathcal{E}_{test}$, where $\mathcal{E}_{train}$ comprises entities appearing in $\mathcal{K}'_G$ and $\mathcal{E}_{test}$ consists of entities present in the test dataset. A bijective mapping function $\Phi: \mathcal{E}_{universal} \rightarrow \{0, 1, ..., |\mathcal{E}_{universal}|-1\}$ is implemented to assign canonical integer identifiers to each entity, facilitating efficient indexing and dimensional reduction.

The canonicalization process ensures consistent entity representation across both KG construction and recommendation evaluation, mitigating potential entity alignment issues and optimizing computational efficiency. The resulting entity mapping $\mathcal{M}_E = \{(e, \Phi(e)) | e \in \mathcal{E}_{universal}\}$ enables seamless integration of the KG with neural recommendation architectures that typically require numerical entity representations.

This algorithmic framework yields a semantically enriched KG with standardized entity references, establishing the foundation for knowledge-aware recommendation algorithms that can simultaneously leverage collaborative signals and semantic relationships to generate contextualized and interpretable recommendations.

\subsection*{F.3 Processing Detail of Geographic Datasets (GeoKG)}

\begin{table*}[htbpt]
\centering
\caption{Comparison of the performance of different methods on GeoKG. The best performance is marked in \textbf{bold} and the second best performance is \underline{underlined}. The \textsc{GraphOracle}-S means the \textsc{GraphOracle} was train from scratch while \textsc{GraphOracle}-F means the \textsc{GraphOracle} was trained by finetune.}
\resizebox{\textwidth}{!}{
\begin{tabular}{c|ccccccccccccccc}
\hline

Method  & TransE & TransR & TransSH & TransD & ComplEx & DistMult & RotatE & NBFNet & RED-GNN & A\*Net & AdaProp & one-shot-subgraph & ULTRA & \textsc{GraphOracle}-S & \textsc{GraphOracle}-F  \\
\hline
 MRR&  0.013 & 0.032 & 0.019 & 0.021 & 0.075 & 0.084 & 0.093 & 0.425 & 0.486 & 0.473 & 0.493 & 0.473 & 0.528 & \underline{0.539} & \textbf{0.639 }\\
 H@1& 0.028      & 0.064 & 0.039 & 0.042 & 0.088 & 0.097 & 0.104 & 0.403 & 0.443 & 0.428 & 0.450 & 0.428 & 0.486 & \underline{0.493} & \textbf{0.552}   \\
 H@10& 0.046     & 0.088 & 0.048 & 0.053 &0.103& 0.125 &0.176 & 0.534 & 0.566 & 0.564 & 0.573 & 0.586 & 0.627 & \underline{0.654} & \textbf{0.793} \\
\hline
\end{tabular}
}
\label{tab:GeoKG}
\end{table*}

\begin{algorithm}[ht]
\SetAlgoLined
\SetKwComment{Comment}{\quad// }{}
\caption{Relation-Balanced Pruning \& Partitioning for Geographic KG}
\KwIn{KG $\mathcal{G}=(\mathcal{E},\mathcal{R},\mathcal{T})$, prune ratio $\rho$, visibility ratio $\theta$, split ratios $\boldsymbol{\alpha}=(\alpha_1,\alpha_2,\alpha_3)$, weight $w$}
\KwOut{Splits $\{\mathcal{G}_i\}_{i=1}^3$, entity map $\pi_{\mathcal{E}}$, relation map $\pi_{\mathcal{R}}$}

\tcc*[f]{Phase I: Relation-Aware Pruning}\;
$\mathcal{T}\leftarrow\textsc{Unique}(\mathcal{T})$\;
Group $\mathcal{T}$ by relation: $\mathcal{T}_r$\;
Compute normalized degree $\bar{\bm d}(v)=\bm d(v)\big/\max_{u}\bm d(u)$\;
$\mathcal{T}^{\rho}\leftarrow\emptyset$\;
\ForEach{$r\in\mathcal{R}$}{
    \ForEach{$(h,r,t)\in\mathcal{T}_r$}{
        $\Upsilon(h,r,t)\leftarrow w\cdot\frac{\bar{\bm d}(h)+\bar{\bm d}(t)}{2}$\;
    }
    $\mathcal{T}^{\rho}\leftarrow\mathcal{T}^{\rho}\cup\textsc{Top}_{\rho}(\mathcal{T}_r,\Upsilon)$\;
}
Define pruned KG $\mathcal{G}^{\rho}$ from $(\mathcal{E}^{\rho},\mathcal{R}^{\rho},\mathcal{T}^{\rho})$\;

\tcc*[f]{Phase II: Visibility Partitioning}\;
Randomly split $\mathcal{E}^{\rho}$ and $\mathcal{R}^{\rho}$ into seen/unseen by $\theta$\;
$\mathcal{T}_{\text{train}}\leftarrow$ triples whose $h,t,r$ are all seen\;
$\mathcal{T}_{\text{eval}}\leftarrow\mathcal{T}^{\rho}\setminus\mathcal{T}_{\text{train}}$\;

\tcc*[f]{Phase III: Distribution Enforcement}\;
Target $|\mathcal{T}_{\text{train}}|=\alpha_1|\mathcal{T}^{\rho}|$\;
Balance $\mathcal{T}_{\text{train}}$ and $\mathcal{T}_{\text{eval}}$ via random moves\;
Split $\mathcal{T}_{\text{eval}}$ into validation/test by $(\alpha_2,\alpha_3)$, ensuring disjointness\;

\tcc*[f]{Phase IV: Finalization}\;
Form $\mathcal{G}_1=(\mathcal{E}^{\text{train}},\mathcal{R}^{\text{train}},\mathcal{T}_{\text{train}})$,  
\hspace*{1.5em}$\mathcal{G}_2=(\mathcal{E}^{\text{valid}},\mathcal{R}^{\text{valid}},\mathcal{T}_{\text{valid}})$,  
\hspace*{1.5em}$\mathcal{G}_3=(\mathcal{E}^{\text{test}},\mathcal{R}^{\text{test}},\mathcal{T}_{\text{test}})$\;
Build index maps $\pi_{\mathcal{E}},\pi_{\mathcal{R}}$ by ascending order of IDs\;
\KwRet{$\{\mathcal{G}_i\}_{i=1}^3,\;\pi_{\mathcal{E}},\;\pi_{\mathcal{R}}$}
\end{algorithm}

\subsubsection{Fully-Inductive Geographic KG Dataset Construction.}
We propose a comprehensive framework for constructing, pruning, and partitioning geographic KGs, designed to ensure relation balance, semantic diversity, and full inductiveness while preserving critical structural information. Our framework systematically addresses five major challenges: (i) maintaining relation diversity, (ii) preserving structural integrity, (iii) controlling visibility of entities and relations, (iv) enforcing strict data partitioning, and (v) ensuring data integrity via rigorous duplicate prevention.

The core innovation lies in the \textit{relation-balanced pruning} strategy introduced in Phase I. Instead of applying a global importance metric across all triples, we stratify the pruning process by relation type. For each relation $r \in \mathcal{R}$, we select the top $\rho=7.5\%$ most important triples based on a relation-specific scoring function $\Upsilon_r$, which estimates the structural importance of triples involving entities $h$ and $t$ via:
\begin{equation}
    \Upsilon_r(h,r,t) = w \cdot \frac{\bar{\bm{d}}(h) + \bar{\bm{d}}(t)}{2}
\end{equation}
where $\bar{\bm{d}}(\cdot)$ denotes the normalized degree of an entity. This ensures that the pruned graph $\mathcal{G}^{\rho}$ maintains balanced semantic representation across both frequent and rare relations, avoiding dominance by high-frequency edges.

To guarantee inductiveness, Phase II performs \textit{visibility-controlled partitioning} by randomly assigning 70\% of entities and relations to the training set. All training triples are composed solely of these “seen” elements, while validation and test triples each include at least one “unseen” entity or relation. This design ensures a fully-inductive setup, where no inference triple shares entities or relations with the training set, thereby enabling robust assessment of generalization to entirely new graph components.

Phase III enforces the 80\%-10\%-10\% train-validation-test ratio by adjusting assignments from Phase II when necessary, while strictly ensuring that each triple appears in only one split. Phase IV finalizes the dataset by constructing the three graph partitions and generating sequential ID mappings for all entities and relations.

Overall, our framework yields a semantically diverse, structurally meaningful, and fully-inductive geographic KG dataset that is reduced to 7.5\% of the original size. The relation-aware pruning and controlled visibility mechanisms work in tandem to ensure both data compactness and inductive generalization capability, while robust duplicate handling preserves data integrity throughout the pipeline.

\section*{G  Complete Experimental Results} \label{Complete Experimental Results}

\noindent In this section, we report the comprehensive experimental results of our study.  
Table~\ref{tab:transductive} presents the performance of \textsc{GraphOracle} on transductive benchmarks, while Table~\ref{tab:entity_inductive}, Table~\ref{tab:fully-inductive} and Table~\ref{tab:fully-inductive1} summarize the results on entity-inductive and fully-inductive settings, respectively.  
Cross-domain evaluations are provided in Table~\ref{tab:primekg}, Table~\ref{tab:Amazon-book}, and Table~\ref{tab:GeoKG}.  
Furthermore, Table~\ref{tab:GraphOracle_pro} illustrates the enhanced performance of \textsc{GraphOracle+} when external information is incorporated.  
Across all datasets and evaluation scenarios, \textsc{GraphOracle} consistently outperforms existing baselines by a notable margin, highlighting the robustness and effectiveness of our proposed approach.

\begin{table*}[htbpt]
\centering
\caption{Comparison of \textsc{GraphOracle} with other reasoning methods in the transductive setting. Best performance is indicated by the \textbf{bold} face numbers, and the \underline{underline} means the second best. ``--'' means unavailable results} 
\resizebox{\textwidth}{!}{
\begin{tabular}{lc|ccc|ccc|ccc|ccc} 
\hline
\multirow{2}{*}{Type} & \multirow{2}{*}{Model} & \multicolumn{3}{c|}{WN18RR} & \multicolumn{3}{c|}{FB15k237} & \multicolumn{3}{c|}{NELL-995} & \multicolumn{3}{c}{YAGO3-10} \\
 & & MRR & H@1 & H@10 & MRR & H@1 & H@10 & MRR & H@1 & H@10 & MRR & H@1 & H@10 \\
\hline
\multirow{10}{*}{Non-GNN}
 & ConvE               & 0.427 & 39.2 & 49.8 & 0.325 & 23.7 & 50.1 & 0.511 & 44.6 & 61.9 & 0.520 & 45.0 & 66.0 \\
 & QuatE               & 0.480 & 44.0 & 55.1 & 0.350 & 25.6 & 53.8 & 0.533 & 46.6 & 64.3 & 0.379 & 30.1 & 53.4 \\
 & RotatE              & 0.477 & 42.8 & 57.1 & 0.337 & 24.1 & 53.3 & 0.508 & 44.8 & 60.8 & 0.495 & 40.2 & 67.0 \\
 & MINERVA             & 0.448 & 41.3 & 51.3 & 0.293 & 21.7 & 45.6 & 0.513 & 41.3 & 63.7 & --     & --    & --    \\
 & DRUM                & 0.486 & 42.5 & 58.6 & 0.343 & 25.5 & 51.6 & 0.532 & 46.0 & 66.2 & 0.531 & 45.3 & 67.6 \\
 & AnyBURL             & 0.471 & 44.1 & 55.2 & 0.301 & 20.9 & 47.3 & 0.398 & 27.6 & 45.4 & 0.542 & 47.7 & 67.3 \\
 & RNNLogic            & 0.483 & 44.6 & 55.8 & 0.344 & 25.2 & 53.0 & 0.416 & 36.3 & 47.8 & 0.554 & 50.9 & 62.2 \\
 & RLogic              & 0.477 & 44.3 & 53.7 & 0.310 & 20.3 & 50.1 & 0.416 & 25.2 & 50.4 & 0.360 & 25.2 & 50.4 \\
 & DuASE               & 0.489 & 44.8 & 56.9 & 0.329 & 23.5 & 51.9 & 0.423 & 37.2 & 59.2 & 0.473 & 38.7 & 62.8 \\
 & GraphRulRL          & 0.483 & 44.6 & 54.1 & 0.385 & 31.4 & 57.5 & 0.425 & 27.8 & 52.7 & 0.432 & 35.4 & 51.7 \\
\hline
\multirow{10}{*}{GNNs}
 & CompGCN             & 0.479 & 44.3 & 54.6 & 0.355 & 26.4 & 53.5 & 0.463 & 38.3 & 59.6 & 0.421 & 39.2 & 57.7 \\
 & NBFNet              & 0.551 & 49.7 & 66.6 & 0.415 & 32.1 & \underline{59.9} & 0.525 & 45.1 & 63.9 & 0.550 & 47.9 & 68.6 \\
 & RED-GNN             & 0.533 & 48.5 & 62.4 & 0.374 & 28.3 & 55.8 & 0.543 & 47.6 & 65.1 & 0.559 & 48.3 & 68.9 \\
 & A*Net               & 0.549 & 49.5 & 65.9 & 0.411 & 32.1 & 58.6 & 0.549 & \underline{48.6} & 65.2 & 0.563 & 49.8 & 68.6 \\
 & AdaProp             & \underline{0.562} & 49.9 & 67.1 & \underline{0.417} & 33.1 & 58.5 & \underline{0.554} & 49.3 & 65.5 & 0.573 & 51.0 & 68.5 \\
 & ULTRA               & 0.480 & 47.9 & 61.4 & 0.368 & \underline{33.9} & 56.4 & 0.509 & 46.2 & 66.0 & 0.557 & 53.1 & 71.0 \\
 & one-shot-subgraph   & 0.567 & \underline{51.4} & 66.6 & 0.304 & 22.3 & 45.4 & 0.547 & 48.5 & 65.1 & \underline{0.606} & \underline{54.0} & \underline{72.1} \\
 & TRIX                & 0.514 & 48.1 & 61.1 & 0.366 & 32.5 & 55.9 & 0.506 & 44.2 & 64.8 & 0.541 & 47.3 & 70.2 \\
 & KG-ICL              & 0.536 & 49.6 & 63.7 & 0.376 & 32.7 & 53.8 & 0.534 & 46.7 & \underline{67.2} & 0.545 & 47.4 & 68.8 \\
  \cline{2-13} 
&  \textbf{\textsc{GraphOracle}} & \textbf{0.675} & \textbf{61.7} & \textbf{76.2} & \textbf{0.471} & \textbf{39.6} & \textbf{66.4} & \textbf{0.621} & \textbf{56.3} & \textbf{75.1} & \textbf{0.696} & \textbf{67.2} & \textbf{80.7} \\
\hline
\end{tabular}
}
\label{tab:transductive}
\end{table*}

\begin{table*}[htbpt]
\centering
\caption{Comparison of \textsc{GraphOracle} with other reasoning methods in the entity inductive setting. Best performance is indicated by the \textbf{bold} face numbers, and the \underline{underline} means the second best.}
\resizebox{\textwidth}{!}{
\begin{tabular}{ll|cccc|cccc|cccc}
\toprule
& & \multicolumn{4}{c|}{WN18RR} & \multicolumn{4}{c|}{FB15k-237} & \multicolumn{4}{c}{NELL-995} \\
& Models & V1 & V2 & V3 & V4 & V1 & V2 & V3 & V4 & V1 & V2 & V3 & V4 \\
\midrule
\multirow{12}{*}{MRR} & RuleN & 0.668 & 0.645 & 0.368 & 0.624 & 0.363 & 0.433 & 0.439 & 0.429 & 0.615 & 0.385 & 0.381 & 0.333 \\
& Neural LP & 0.649 & 0.635 & 0.361 & 0.628 & 0.325 & 0.389 & 0.400 & 0.396 & 0.610 & 0.361 & 0.367 & 0.261 \\
& DRUM & 0.666 & 0.646 & 0.380 & 0.627 & 0.333 & 0.395 & 0.402 & 0.410 & 0.628 & 0.365 & 0.375 & 0.273 \\
& GraIL & 0.627 & 0.625 & 0.323 & 0.553 & 0.279 & 0.276 & 0.251 & 0.227 & 0.481 & 0.297 & 0.322 & 0.262 \\
& CoMPILE & 0.577 & 0.578 & 0.308 & 0.548 & 0.287 & 0.276 & 0.262 & 0.213 & 0.330 & 0.248 & 0.319 & 0.229 \\
& NBFNet & 0.684 & 0.652 & 0.425 & 0.604 & 0.307 & 0.369 & 0.331 & 0.305 & 0.584 & 0.410 & 0.425 & 0.287 \\
& RED-GNN & 0.701 & 0.690 & 0.427 & 0.651 & 0.369 & 0.469 & 0.445 & 0.442 & 0.637 & 0.419 & 0.436 & 0.363 \\
& AdaProp & 0.733 & 0.715 & 0.474 & 0.662 & 0.310 & 0.471 & 0.471 & 0.454 & 0.644 & 0.452 & 0.435 & 0.366 \\
& ULTRA & 0.685 & 0.679 & 0.411 & 0.614 & 0.509 & 0.524 & 0.504 & 0.496 & 0.757 & 0.575 & 0.563 & 0.469 \\
& TRIX & 0.705 & 0.682 & 0.425 & 0.650 & 0.515 & 0.525 & 0.501 & 0.493 & 0.804 & 0.571 & 0.571 & 0.551 \\
& KG-ICL & \underline{0.762} & \underline{0.721} & \underline{0.503} & \underline{0.683} & \underline{0.531} & \underline{0.568} & \underline{0.537} & \underline{0.525} & \underline{0.841} & \underline{0.641} & \underline{0.631} & \underline{0.594} \\
&\textbf{\textsc{GraphOracle}} & \textbf{0.807} & \textbf{0.793} & \textbf{0.569} & \textbf{0.762} & \textbf{0.619} & \textbf{0.631} & \textbf{0.694} & \textbf{0.658} & \textbf{0.864} & \textbf{0.684} & \textbf{0.659} & \textbf{0.619} \\
\midrule
\multirow{12}{*}{Hit@1 (\%)} & RuleN & 63.5 & 61.1 & 34.7 & 59.2 & 30.9 & 34.7 & 34.5 & 33.8 & 54.5 & 30.4 & 30.3 & 24.8 \\
& Neural LP & 59.2 & 57.5 & 30.4 & 58.3 & 24.3 & 28.6 & 30.9 & 28.9 & 50.0 & 24.9 & 26.7 & 13.7 \\
& DRUM & 61.3 & 59.5 & 33.0 & 58.6 & 24.7 & 28.4 & 30.8 & 30.9 & 50.0 & 27.1 & 26.2 & 16.3 \\
& GraIL & 55.4 & 54.2 & 27.8 & 44.3 & 20.5 & 20.2 & 16.5 & 14.3 & 42.5 & 19.9 & 22.4 & 15.3 \\
& CoMPILE & 47.3 & 48.5 & 25.8 & 47.3 & 20.8 & 17.8 & 16.6 & 13.4 & 10.5 & 15.6 & 22.6 & 15.9 \\
& NBFNet & 59.2 & 57.5 & 30.4 & 57.4 & 19.0 & 22.9 & 20.6 & 18.5 & 50.0 & 27.1 & 26.2 & 23.3 \\
& RED-GNN & 65.3 & 63.3 & \underline{36.8} & 60.6 & 30.2 & 38.1 & 35.1 & 34.0 & 52.5 & 31.9 & 34.5 & 25.9 \\
& AdaProp &\underline{66.8} & \underline{64.2} & \underline{39.6} & \underline{61.1} & 19.1 & 37.2 & 37.7 & 35.3 & 52.2 & 34.4 & 33.7 & 24.7 \\
& ULTRA & 61.5 & 58.7 & 33.5 & 58.7 & 32.2 & 39.9 & 40.5 & 37.2 & 50.7 & 35.8 & 36.4 & 28.8 \\
& TRIX & 63.9 & 58.4 & 34.7 & 59.3 & 32.9 & 39.8 & 40.7 & 37.0 & 53.9 & 35.5 & 36.9 & 31.7 \\
& KG-ICL & 65.4 & 61.7 & 36.9 & 60.5 & \underline{41.1} & \underline{43.8} & \underline{42.6} & \underline{39.6} & \underline{59.4} & \underline{39.6} & \underline{41.2} & \underline{35.6} \\
&\textbf{\textsc{GraphOracle}} & \textbf{76.8} & \textbf{79.8} & \textbf{47.3} & \textbf{67.8} & \textbf{40.4} & \textbf{51.9} & \textbf{52.8} & \textbf{50.9} & \textbf{65.6} & \textbf{52.6} & \textbf{51.3} & \textbf{44.9} \\
\midrule
\multirow{12}{*}{Hit@10 (\%)} & RuleN & 73.0 & 69.4 & 40.7 & 68.1 & 44.6 & 59.9 & 60.0 & 60.5 & 76.0 & 51.4 & 53.1 & 48.4 \\
& Neural LP & 77.2 & 74.9 & 47.6 & 70.6 & 46.8 & 58.6 & 57.1 & 59.3 & 87.1 & 56.4 & 57.6 & 53.9 \\
& DRUM & 77.7 & 74.7 & 47.7 & 70.2 & 47.4 & 59.5 & 57.1 & 59.3 & 87.3 & 54.0 & 57.7 & 53.1 \\
& GraIL & 76.0 & 77.6 & 40.9 & 68.7 & 42.9 & 42.4 & 42.4 & 38.9 & 56.5 & 49.6 & 51.8 & 50.6 \\
& CoMPILE & 74.7 & 74.3 & 40.6 & 67.0 & 43.9 & 45.7 & 44.9 & 35.8 & 57.5 & 44.6 & 51.5 & 42.1 \\
& NBFNet & 82.7 & 79.9 & 56.3 & 70.2 & 51.7 & 63.9 & 58.8 & 55.9 & 79.5 & 63.5 & 60.6 & 59.1 \\
& RED-GNN & 79.9 & 78.0 & 52.4 & 72.1 & 48.3 & 62.9 & 60.3 & 62.1 & 86.6 & 60.1 & 59.4 & 55.6 \\
& AdaProp & \underline{86.6} & \underline{83.6} & \underline{62.6} & \underline{75.5} & 55.1 & 65.9 & 63.7 & 63.8 & 88.6 & 65.2 & 61.8 & 60.7 \\
& ULTRA & 79.3 & 77.9 & 54.6 & 72.0 & 67.0 & 71.0 & 66.3 & 68.4 & 87.8 & 76.1 & 75.5 & 73.3 \\
& TRIX & 79.8 & 78.0 & 54.3 & 72.2 & 68.2 & 73.0 & 69.9 & 68.7 & 89.9 & 76.4 & 75.9 & 77.2 \\
& KG-ICL & 82.7 & 78.7 & \underline{62.6} & 74.9 & \underline{70.0} & \underline{76.8} & \underline{70.4} & \underline{70.6} & \textbf{99.5} & \underline{83.5} & \underline{79.9} & \underline{80.2} \\
&\textbf{\textsc{GraphOracle}} & \textbf{92.3} & \textbf{92.6} & \textbf{69.5} & \textbf{81.6} & \textbf{76.7} & \textbf{79.2} & \textbf{75.7} & \textbf{78.5} & \underline{97.2} & \textbf{86.9} & \textbf{86.2} & \textbf{82.4} \\
\bottomrule
\end{tabular}
}
\label{tab:entity_inductive}
\end{table*}

\begin{table*}[htbp]
\centering
\caption{Comparison of \textsc{GraphOracle} with other reasoning methods in fully-inductive setting. Best performance is indicated by the \textbf{bold} face numbers, and the \underline{underline} means the second best. H@1'' and H@10'' are short for Hit@1 and Hit@10 (in percentage), respectively. --'' means unavailable results.}
\resizebox{\textwidth}{!}{
\begin{tabular}{l|ccc|ccc|ccc|ccc}
\hline
\multirow{2}{*}{Model} & \multicolumn{3}{c|}{Nell-100} & \multicolumn{3}{c|}{Nell-75} & \multicolumn{3}{c|}{Nell-50} & \multicolumn{3}{c}{Nell-25} \\
 & MRR & H@1 & H@10 & MRR & H@1 & H@10 & MRR & H@1 & H@10 & MRR & H@1 & H@10 \\
\hline
GraIL & 0.135 & 0.114 & 0.173 & 0.096 & 0.056 & 0.205 & 0.162 & 0.104 & 0.288 & 0.216 & 0.160 & 0.366 \\
CoMPILE & 0.123 & 0.071 & 0.209 & 0.178 & 0.093 & 0.361 & 0.194 & 0.125 & 0.330 & 0.189 & 0.115 & 0.324 \\
SNRI & 0.042 & 0.029 & 0.064 & 0.088 & 0.040 & 0.177 & 0.130 & 0.095 & 0.187 & 0.190 & 0.140 & 0.270 \\
INDIGO & 0.160 & 0.109 & 0.247 & 0.121 & 0.098 & 0.156 & 0.167 & 0.134 & 0.217 & 0.166 & 0.134 & 0.206 \\
RMPI & 0.220 & 0.136 & 0.376 & 0.138 & 0.061 & 0.275 & 0.185 & 0.109 & 0.307 & 0.213 & 0.130 & 0.329 \\
CompGCN & 0.008 & 0.001 & 0.014 & 0.014 & 0.003 & 0.025 & 0.003 & 0.000 & 0.005 & 0.006 & 0.000 & 0.010 \\
NodePiece & 0.012 & 0.004 & 0.018 & 0.042 & 0.020 & 0.081 & 0.037 & 0.013 & 0.079 & 0.098 & 0.057 & 0.166 \\
NeuralLP & 0.084 & 0.035 & 0.181 & 0.117 & 0.048 & 0.273 & 0.101 & 0.064 & 0.190 & 0.148 & 0.101 & 0.271 \\
DRUM & 0.076 & 0.044 & 0.138 & 0.152 & 0.072 & 0.313 & 0.107 & 0.070 & 0.193 & 0.161 & 0.119 & 0.264 \\
BLP & 0.019 & 0.006 & 0.037 & 0.051 & 0.012 & 0.120 & 0.041 & 0.011 & 0.093 & 0.049 & 0.024 & 0.095 \\
QBLP & 0.004 & 0.000 & 0.003 & 0.040 & 0.007 & 0.095 & 0.048 & 0.020 & 0.097 & 0.073 & 0.027 & 0.151 \\
NBFNet & 0.096 & 0.032 & 0.199 & 0.137 & 0.077 & 0.255 & 0.225 & 0.161 & 0.346 & 0.283 & 0.224 & 0.417 \\
RED-GNN & 0.212 & 0.114 & 0.385 & 0.203 & 0.129 & 0.353 & 0.179 & 0.115 & 0.280 & 0.214 & 0.166 & 0.266 \\
RAILD & 0.018 & 0.005 & 0.037 & -- & -- & -- & -- & -- & -- & -- & -- & -- \\
INGRAM & 0.309 & 0.212 & 0.506 & 0.261 & 0.167 & 0.464 & 0.281 & 0.193 & 0.453 & 0.334 & 0.241 & 0.501 \\
ULTRA & 0.458 & 0.423 & 0.684 & 0.374 & 0.369 & 0.570 & 0.418 & 0.256 & 0.595 & 0.407 & 0.278 & 0.596 \\
TRIX & 0.482 & 0.437 & 0.691 & 0.351 & 0.325 & 0.525 & 0.405 & 0.213 & 0.555 & 0.377 & 0.262 & 0.589 \\
KG-ICL & \underline{0.557} & \underline{0.459} & \underline{0.766} & \underline{0.446} & \underline{0.378} & \underline{0.681} & \underline{0.528} & \underline{0.274} & \underline{0.708} & \underline{0.540} & \underline{0.301} & \underline{0.730} \\
\textbf{\textsc{GraphOracle}} & \textbf{0.702} & \textbf{0.623} & \textbf{0.905} & \textbf{0.612} & \textbf{0.423} & \textbf{0.923} & \textbf{0.589} & \textbf{0.421} & \textbf{0.868} & \textbf{0.579} & \textbf{0.389} & \textbf{0.923} \\
\hline
\hline
\multirow{2}{*}{Model} & \multicolumn{3}{c|}{WK-100} & \multicolumn{3}{c|}{WK-75} & \multicolumn{3}{c|}{WK-50} & \multicolumn{3}{c}{WK-25} \\
 & MRR & H@1 & H@10 & MRR & H@1 & H@10 & MRR & H@1 & H@10 & MRR & H@1 & H@10 \\
\hline
CompGCN & 0.003 & 0.000 & 0.009 & 0.015 & 0.003 & 0.028 & 0.003 & 0.001 & 0.002 & 0.009 & 0.000 & 0.020 \\
NodePiece & 0.007 & 0.002 & 0.018 & 0.021 & 0.003 & 0.052 & 0.008 & 0.002 & 0.013 & 0.053 & 0.019 & 0.122 \\
NeuralLP & 0.009 & 0.005 & 0.016 & 0.020 & 0.004 & 0.054 & 0.025 & 0.007 & 0.054 & 0.068 & 0.046 & 0.104 \\
DRUM & 0.010 & 0.004 & 0.019 & 0.020 & 0.007 & 0.043 & 0.017 & 0.002 & 0.046 & 0.064 & 0.035 & 0.116 \\
BLP & 0.012 & 0.003 & 0.025 & 0.043 & 0.016 & 0.089 & 0.041 & 0.013 & 0.092 & 0.125 & 0.055 & 0.283 \\
QBLP & 0.012 & 0.003 & 0.025 & 0.044 & 0.016 & 0.091 & 0.035 & 0.011 & 0.080 & 0.116 & 0.042 & 0.294 \\
NBFNet & 0.014 & 0.005 & 0.026 & 0.072 & 0.028 & 0.172 & 0.062 & 0.036 & 0.105 & 0.154 & 0.092 & 0.301 \\
RED-GNN & 0.096 & 0.070 & 0.136 & 0.172 & 0.110 & 0.290 & 0.058 & 0.033 & 0.093 & 0.170 & 0.111 & 0.263 \\
RAILD & 0.026 & 0.010 & 0.052 & -- & -- & -- & -- & -- & -- & -- & -- & -- \\
INGRAM & 0.107 & 0.072 & 0.169 & 0.247 & 0.179 & 0.362 & 0.068 & 0.034 & 0.135 & 0.186 & 0.124 & 0.309 \\
ULTRA & 0.168 & 0.089 & 0.286 & 0.380 & 0.278 & \underline{0.635} & 0.140 & 0.076 & 0.280 & 0.321 & 0.388 & 0.535 \\
TRIX & 0.188 & 0.093 & 0.290 & 0.368 & 0.254 & 0.513 & 0.166 & 0.078 & 0.313 & 0.300 & 0.354 & 0.401 \\
KG-ICL & \underline{0.270} & \underline{0.127} & \underline{0.415} & \underline{0.466} & \underline{0.313} & 0.626 & \underline{0.277} & \underline{0.091} & \underline{0.432} & \underline{0.425} & \underline{0.434} & \underline{0.628} \\
\textbf{\textsc{GraphOracle}} & \textbf{0.417} & \textbf{0.192} & \textbf{0.711} & \textbf{0.469} & \textbf{0.345} & \textbf{0.698} & \textbf{0.362} & \textbf{0.101} & \textbf{0.498} & \textbf{0.582} & \textbf{0.460} & \textbf{0.780} \\
\hline

\hline
\end{tabular}
}
\label{tab:fully-inductive}
\end{table*}

\begin{table*}[htbp]
\centering
\caption{Comparison of \textsc{GraphOracle} with other reasoning methods in fully-inductive setting. Best performance is indicated by the \textbf{bold} face numbers, and the \underline{underline} means the second best. H@1'' and H@10'' are short for Hit@1 and Hit@10 (in percentage), respectively. --'' means unavailable results.}
\resizebox{\textwidth}{!}{
\begin{tabular}{l|ccc|ccc|ccc|ccc}
\hline
\multirow{2}{*}{Model} & \multicolumn{3}{c|}{FB-100} & \multicolumn{3}{c|}{FB-75} & \multicolumn{3}{c|}{FB-50} & \multicolumn{3}{c}{FB-25} \\
 & MRR & H@1 & H@10 & MRR & H@1 & H@10 & MRR & H@1 & H@10 & MRR & H@1 & H@10 \\
\hline
CompGCN & 0.015 & 0.008 & 0.025 & 0.013 & 0.000 & 0.026 & 0.004 & 0.002 & 0.006 & 0.003 & 0.000 & 0.004 \\
NodePiece & 0.006 & 0.001 & 0.009 & 0.016 & 0.007 & 0.029 & 0.021 & 0.006 & 0.048 & 0.044 & 0.011 & 0.114 \\
NeuralLP & 0.026 & 0.007 & 0.057 & 0.056 & 0.030 & 0.099 & 0.088 & 0.043 & 0.184 & 0.164 & 0.098 & 0.309 \\
DRUM & 0.034 & 0.011 & 0.077 & 0.065 & 0.034 & 0.121 & 0.101 & 0.061 & 0.191 & 0.175 & 0.109 & 0.320 \\
BLP & 0.017 & 0.004 & 0.035 & 0.047 & 0.024 & 0.085 & 0.078 & 0.037 & 0.156 & 0.107 & 0.053 & 0.212 \\
QBLP & 0.013 & 0.003 & 0.026 & 0.041 & 0.017 & 0.084 & 0.071 & 0.030 & 0.147 & 0.104 & 0.043 & 0.226 \\
NBFNet & 0.072 & 0.026 & 0.154 & 0.089 & 0.048 & 0.166 & 0.130 & 0.071 & 0.259 & 0.224 & 0.137 & 0.410 \\
RED-GNN & 0.121 & 0.053 & 0.263 & 0.107 & 0.057 & 0.201 & 0.129 & 0.072 & 0.251 & 0.145 & 0.077 & 0.284 \\
RAILD & 0.031 & 0.016 & 0.048 & -- & -- & -- & -- & -- & -- & -- & -- & -- \\
INGRAM & 0.223 & 0.146 & 0.371 & 0.189 & 0.119 & 0.325 & 0.117 & 0.067 & 0.218 & 0.133 & 0.067 & 0.271 \\
ULTRA & 0.444 & 0.287 & 0.643 & 0.400 & 0.269 & 0.598 & 0.334 & 0.275 & 0.538 & 0.383 & 0.242 & 0.635 \\
TRIX & 0.436 & 0.269 & 0.633 & 0.401 & 0.263 & 0.611 & 0.334 & 0.277 & 0.547 & 0.393 & 0.256 & 0.650 \\
KG-ICL & \underline{0.499} & \underline{0.307} & \underline{0.719} & \underline{0.458} & \underline{0.274} & \underline{0.664} & \underline{0.384} & \underline{0.291} & \underline{0.598} & \underline{0.434} & \underline{0.279} & \underline{0.694} \\
\textbf{\textsc{GraphOracle}} & \textbf{0.576} & \textbf{0.407} & \textbf{0.812} & \textbf{0.538} & \textbf{0.356} & \textbf{0.872} & \textbf{0.585} & \textbf{0.434} & \textbf{0.913} & \textbf{0.562} & \textbf{0.370} & \textbf{0.930} \\
\hline
\end{tabular}
}
\label{tab:fully-inductive1}
\end{table*}

\begin{table*}[htbp]
\centering
\caption{Performance comparison among ULTRA, TRIX, KG-ICL, and GraphOracle across different datasets. Best results are in \textbf{bold} and second best are \underline{underlined}.}
\resizebox{\textwidth}{!}{
\begin{tabular}{l l c c c c c c c c}
\toprule
Type & Model & \multicolumn{2}{c}{ULTRA} & \multicolumn{2}{c}{TRIX} & \multicolumn{2}{c}{KG-ICL} & \multicolumn{2}{c}{GraphOracle} \\
& & MRR & Hit@10 & MRR & Hit@10 & MRR & Hit@10 & MRR & Hit@10 \\
\midrule
\multirow{10}{*}{Transductive} 
& CoDEx Small & \underline{0.490} & \underline{0.686} & 0.484 & 0.676 & 0.479 & 0.662 & \textbf{0.512} & \textbf{0.697} \\
& CoDEx Medium & 0.372 & 0.525 & 0.365 & 0.521 & \underline{0.402} & \underline{0.565} & \textbf{0.417} & \textbf{0.574} \\
& CoDEx Large & 0.343 & 0.478 & \underline{0.388} & 0.481 & 0.388 & \underline{0.508} & \textbf{0.396} & \textbf{0.523} \\
& WDsinger & 0.417 & 0.526 & \underline{0.502} & \underline{0.620} & 0.493 & 0.599 & \textbf{0.512} & \textbf{0.654} \\
& NELL23k & 0.268 & 0.450 & \underline{0.306} & \underline{0.536} & 0.329 & 0.552 & \textbf{0.333} & \textbf{0.572} \\
& FB15k237\_10 & 0.254 & 0.411 & \underline{0.253} & \underline{0.408} & 0.260 & 0.416 & \textbf{0.269} & \textbf{0.435} \\
& FB15k237\_20 & 0.274 & 0.445 & \underline{0.273} & \underline{0.441} & 0.284 & 0.456 & \textbf{0.297} & \textbf{0.482} \\
& FB15k237\_50 & \underline{0.325} & \underline{0.528} & 0.322 & 0.522 & 0.324 & 0.499 & \textbf{0.336} & \textbf{0.541} \\
& DBpedia100k & \underline{0.436} & \underline{0.603} & 0.457 & 0.619 & 0.455 & 0.604 & \textbf{0.479} & \textbf{0.643} \\
& AristoV4 & 0.343 & 0.496 & \underline{0.345} & \underline{0.499} & 0.313 & 0.480 & \textbf{0.374} & \textbf{0.524} \\
& ConceptNet100k & 0.310 & 0.529 & \underline{0.340} & \underline{0.564} & 0.371 & 0.584 & \textbf{0.386} & \textbf{0.602} \\
\midrule
\multirow{7}{*}{Entity Inductive}
& Hetionet & \underline{0.399} & \underline{0.538} & 0.394 & 0.534 & 0.269 & 0.402 & \textbf{0.417} & \textbf{0.556} \\
& ILPC Small & 0.303 & \underline{0.453} & \underline{0.310} & 0.455 & 0.316 & \underline{0.473} & \textbf{0.339} & \textbf{0.497} \\
& ILPC Large & 0.308 & \underline{0.431} & \underline{0.310} & 0.431 & 0.295 & 0.411 & \textbf{0.345} & \textbf{0.451} \\
& HM 1k & 0.042 & 0.100 & \underline{0.072} & \underline{0.128} & 0.089 & 0.144 & \textbf{0.097} & \textbf{0.178} \\
& HM 3k & 0.030 & 0.090 & \underline{0.069} & \underline{0.118} & 0.081 & 0.129 & \textbf{0.089} & \textbf{0.143} \\
& HM 5k & 0.025 & 0.068 & \underline{0.074} & \underline{0.118} & 0.070 & 0.108 & \textbf{0.096} & \textbf{0.145} \\
& IndigoBM & \underline{0.432} & \underline{0.639} & 0.436 & 0.645 & 0.440 & 0.641 & \textbf{0.483} & \textbf{0.697} \\
\midrule
\multirow{9}{*}{Fully Inductive}
& MT1 tax & 0.330 & 0.459 & \underline{0.397} & \underline{0.508} & 0.411 & 0.521 & \textbf{0.491} & \textbf{0.568} \\
& MT1 health & 0.380 & 0.467 & \underline{0.376} & \underline{0.457} & 0.387 & 0.479 & \textbf{0.405} & \textbf{0.501} \\
& MT2 org & 0.104 & 0.170 & \underline{0.098} & \underline{0.162} & 0.100 & 0.171 & \textbf{0.132} & \textbf{0.193} \\
& MT2 sci & 0.311 & 0.451 & \underline{0.331} & \underline{0.526} & 0.303 & 0.396 & \textbf{0.337} & \textbf{0.574} \\
& MT3 art & 0.306 & 0.473 & \underline{0.289} & \underline{0.461} & 0.306 & 0.460 & \textbf{0.315} & \textbf{0.481} \\
& MT3 infra & \underline{0.657} & \underline{0.807} & 0.672 & 0.810 & 0.676 & \underline{0.808} & \textbf{0.697} & \textbf{0.829} \\
& MT4 sci & 0.303 & 0.478 & \underline{0.305} & \underline{0.482} & 0.307 & 0.473 & \textbf{0.321} & \textbf{0.496} \\
& MT4 health & \underline{0.704} & \underline{0.785} & 0.702 & 0.785 & 0.710 & 0.776 & \textbf{0.721} & \textbf{0.796} \\
& Metafam & \textbf{0.997} & \textbf{1.000} & \textbf{0.997} & \textbf{1.000} & \textbf{1.000} & \textbf{1.000} & \textbf{1.000} & \textbf{1.000} \\
& FBNELL & \underline{0.481} & \underline{0.661} & 0.478 & 0.655 & 0.516 & 0.699 & \textbf{0.523} & \textbf{0.732} \\
& NL-0 & 0.329 & 0.551 & \underline{0.385} & \underline{0.549} & 0.555 & 0.765 & \textbf{0.566} & \textbf{0.777} \\
\bottomrule
\end{tabular}
}
\end{table*}

\section*{H  Detail Design of \textsc{GraphOracle+}} \label{Updated Methods for Multimodal Reasoning}

\begin{table*}[h]
\centering
\caption{Performance Comparison on PrimeKG: Evaluating \textsc{GraphOracle} Enhanced by External Entity Initialization (\textsc{GraphOracle+})}
\resizebox{\textwidth}{!}{
\begin{tabular}{c|ccc|ccc|ccc|ccc|ccc|ccc|ccc}
\hline
\multirow{2}{*}{Method} 
& \multicolumn{3}{c|}{Protein$\rightarrow$BP} 
& \multicolumn{3}{c|}{Protein$\rightarrow$MF} 
& \multicolumn{3}{c|}{Protein$\rightarrow$CC} 
& \multicolumn{3}{c|}{Drug$\rightarrow$Disease} 
& \multicolumn{3}{c|}{Protein$\rightarrow$Drug} 
& \multicolumn{3}{c|}{Disease$\rightarrow$Protein} 
& \multicolumn{3}{c}{Drug$\not\!\rightarrow$Disease} \\ 
 & MRR & H@1 & H@10 & MRR & H@1 & H@10 & MRR & H@1 & H@10 & MRR & H@1 & H@10
 & MRR & H@1 & H@10 & MRR & H@1 & H@10 & MRR & H@1 & H@10 \\
\hline
\textsc{GraphOracle} & .498 & .323 & .684 & .499 & .366 & .748 & .475 & .325 & .752 & .268 & .175 & .448 & .232 & .167 & .378 & .299 & .187 & .531 & .192 & .135 & .380 \\
 \textbf{\textsc{GraphOracle+}} & 
\textbf{.573} & \textbf{.371} & \textbf{.787} & 
\textbf{.574} & \textbf{.421} & \textbf{.860} & 
\textbf{.546} & \textbf{.374} & \textbf{.865} & 
\textbf{.308} & \textbf{.201} & \textbf{.515} & 
\textbf{.267} & \textbf{.192} & \textbf{.435} & 
\textbf{.344} & \textbf{.215} & \textbf{.611} & 
\textbf{.221} & \textbf{.155} & \textbf{.437} \\
\hline
\end{tabular}
}
\label{tab:GraphOracle_pro}

\end{table*}

In real-world scenarios, users often query models with procedural or temporal “How”-type questions rather than isolated factual prompts. Script-based evaluation frameworks~\cite{li2025scedit} emphasize the importance of integrating external knowledge to support dynamic, multi-step reasoning. Motivated by this, we extend \textsc{GraphOracle} by incorporating modality-specific external features to enhance its representational capacity.

For each entity $e_i$, we update its information as $e_i = \{x^i, c^i\}$ which combines an additional feature vector $x^i$ and a modality tag $c^i$. For example, in PrimeKG, a drug may be defined as $c^i = \text{"drug"}$ and $x^i = \text{"functional description of the drug"}$.
To further enhance \textsc{GraphOracle}, we introduce a methodological extension by integrating external information, resulting in an improved variant termed \textsc{GraphOracle+}.  
Specifically, we incorporate multiple unimodal foundation models (uni-FMs), and by leveraging the embeddings generated by these uni-FMs, we effectively enrich the representations of individual nodes.  
In the current era of large-scale models, the ability to seamlessly integrate heterogeneous sources of information is of paramount importance.
As described, for any two entities $e_i$ and $e_j$ originating from distinct modalities, we utilize modality-specific foundation models to encode their features.  
The initial embedding for entity $e_i$ under a query context $(e_q, r_q)$ is formulated as:
\begin{equation}
    \bm{h}_{e_i}(e_q, r_q) = \psi(x^i, c^i)
\end{equation}
where $\psi(x^i, c^i)$ denotes a modality-specific encoder selected based on the entity type $c^i$, and $x^i$ represents the raw input features of $e_i$.  
To unify embeddings produced by different unimodal encoders into a common representation space, we introduce a \textit{modality-aware projection function} $\mathcal{T}(c^i)$, which aligns each modality to a shared latent space.
$
\bm{h}_{e_i}^0(e_q, r_q) = \mathcal{T}(\bm{h}_{e_i}(e_q, r_q), c^i) \in \mathbb{R}^d,
$
where $c^i$ denotes the modality type of entity $e_i$, and $\mathcal{T}(\cdot, \cdot)$ ensures that all modality-specific outputs are projected into a unified $d$-dimensional space.
Furthermore, to make the model modally-aware, we encode its modality type $c^i$ to obtain its modality embedding $\bm c^i$. The complete encoding process is:

\begin{equation}
\resizebox{0.5\textwidth}{!}{$
\begin{array}{ll}
&  \bm h^{0}_{e}(e_q, r_q,{c^i}) = \mathcal{T}(\bm{h}_{e_i}(e_q, r_q), c^i)=\mathcal{T}(\psi(x^i,c^i), c^i),
\\[0.5em]
& \bm h^{\ell}_{e_o}(e_q, r_q,{c^i}^\ell) = \delta \left( \bm W^{\ell} \cdot \sum_{(e_s, r, e_o) \in \hat{\mathcal E}^{\ell}_{e_q}} \alpha^{\ell}_{e_s, r, e_o|r_q} \left( \bm h^{\ell-1}_{e_s}(e_q, r_q,{c^i}^{\ell-1}) + \Psi ( {\bm c^i}^{\ell-1},  {\bm c^i}^{\ell}, \bm h^{\ell}_{r}) \right) \right),
\end{array}  
$}
\end{equation}

where $ {\alpha} ^\ell _{e_s,r,e_s|r_q} $  is defined the same as Eq.~\eqref{eq:alhpa}, ${\bm c^i}^{\ell-1}$ and ${\bm c^i}^\ell$ is the modality embedding of nodes at $\ell-1$ and $\ell$ respectively, and $\Psi$ is a vanilla six-layer transformer model for bridging different modalities. Table~\ref{tab:GraphOracle_pro} demonstrates the powerful performance of \textsc{GraphOracle+} and proves the scalability of our model.

\section*{I  Theoretical Analysis of the \textsc{GraphOracle} Model} \label{Theoretical Analysis of the GraphOracle Model}

In this section, we provide rigorous theoretical guarantees for the \textsc{GraphOracle} framework, analyzing its expressiveness, generalization capabilities, convergence properties, stability under perturbations, and relation‑dependency correctness.

\subsection*{I.1 Expressiveness and Representation Capacity}

\begin{theorem}[Representation Capacity]
The RDG representation in \textsc{GraphOracle} with L message passing layers can distinguish between any two non-isomorphic relation subgraphs with a maximum path length of L.
\end{theorem}

\begin{proof}
We prove this by induction on the number of message passing layers L.

Base case ($\ell$=1): For $\ell=1$, the representation of relation r after one message passing layer is:
\begin{equation}
\resizebox{0.49\textwidth}{!}{$
    \bm{h}^{\ell}_{r_v \mid r_q} = \sigma\left( \frac{1}{H} \sum_{h=1}^H \left[
\bm{W}_1^{\ell,h} \sum_{r_u \in \mathcal{N}^{\text{past}}(r_v)} \hat \alpha_{r_u r}^{\ell,h} \, \bm{h}^{\ell-1}_{r_u \mid r_q}+ 
\bm{W}_2^{\ell,h} \, \hat \alpha_{r_v r}^{\ell,h} \, \bm{h}^{\ell-1}_{r_v \mid r_q}
\right] \right),
$}
\end{equation}

Since the initial representation $\bm{h}_{r|r_q}^0 = \delta_{r,r_q} \cdot \bm{1}^d$ distinguishes the query relation from all others, and the attention weights $\hat{{\alpha}}_{r_u r}^h$ are distinct for different neighborhood configurations, non-isomorphic relation subgraphs of depth 1 will have distinct representations.

Inductive step: Assume the statement holds for $L=k$. For $L=k+1$, each relation's representation now incorporates information from relations that are $k+1$ steps away. If two relation subgraphs are non-isomorphic within $k+1$ steps, either:
\begin{enumerate}
    \item They were already non-isomorphic within $k$ steps, which by our inductive hypothesis leads to different representations, or The difference occurs exactly at step $k+1$, which will result in different inputs to the $(k+1)$-th layer message passing function, thus producing different representations.
\end{enumerate}

Therefore, the theorem holds for all $L$.
\end{proof}

\begin{theorem}[Expressive Power]
For any continuous function $f: \mathcal{X} \to \mathcal{Y}$ on compact sets $\mathcal{X}$ and $\mathcal{Y}$, there exists a \textsc{GraphOracle} model with sufficient width and depth that can approximate $f$ with arbitrary precision.
\end{theorem}

\begin{proof}
The proof leverages the universal approximation theorem for neural networks. Our model consists of three components:

1. The RDG representation module:
   \begin{equation}
   \resizebox{0.49\textwidth}{!}{$
    \bm{h}^{\ell}_{r_v \mid r_q} = \sigma\left( \frac{1}{H} \sum_{h=1}^H \left[
\bm{W}_1^{\ell,h} \sum_{r_u \in \mathcal{N}^{\text{past}}(r_v)} \hat \alpha_{r_u r_v}^{\ell,h} \, \bm{h}^{\ell-1}_{r_u \mid r_q}+ 
\bm{W}_2^{\ell,h} \, \hat \alpha_{r_v r_v}^{\ell,h} \, \bm{h}^{\ell-1}_{r_v \mid r_q}
\right] \right),
$}
   \end{equation}

2. The universal entity representation module:
   \begin{equation}
    \bm h^{\ell}_{e|q} = \delta \left( \bm W^{\ell} \cdot \sum_{(e_s, r, e) \in \mathcal F_\text{train}} \alpha^{\ell}_{e_s, r|r_q} \left( \bm h^{\ell-1}_{e_s|q} + \bm h^{L_r}_{r|r_q} \right) \right),
       \end{equation}

3. The scoring function:
   \begin{equation}
       s(e_q, r_q, e_a) = \bm{w}_s^{\top}\bm h_{r_q}^{L}(e_q, e_a)
   \end{equation}

Each component is constructed from differentiable functions that can be approximated by neural networks with sufficient capacity. By the universal approximation theorem, for any continuous function $f$ and any $\epsilon > 0$, there exists a neural network that approximates $f$ within an error bound of $\epsilon$.

Therefore, with sufficient width (embedding dimension $d$) and depth (number of layers $L$), \textsc{GraphOracle} can approximate any continuous function over the KG with arbitrary precision.
\end{proof}

\subsection*{I.2 Generalization Bounds}

\begin{theorem}[Generalization Error Bound]
For a \textsc{GraphOracle} model with parameters $\Theta$ trained on a dataset $\mathcal{D}$ with $N$ triples sampled from a KG with $|\mathcal{V}|$ entities and $|\mathcal{R}|$ relations, the expected generalization error is bounded by:
\begin{equation}
    \mathbb{E}[\mathcal{L}_{test}(\Theta) - \mathcal{L}_{train}(\Theta)] \leq \mathcal{O}\left(\sqrt{\frac{\log(|\mathcal{V}| \cdot |\mathcal{R}|)}{N}}\right)
\end{equation} 
\end{theorem}

\begin{proof}
Let $\mathcal{H}$ be the hypothesis class of all possible \textsc{GraphOracle} models with fixed architecture. The VC-dimension of $\mathcal{H}$ can be bounded by $\mathcal{O}(p \log p)$, where $p$ is the number of parameters in the model, which is proportional to $|\mathcal{R}| \cdot d^2 \cdot L$.

By standard results from statistical learning theory, the generalization error is bounded by:
\begin{equation}
    \mathbb{E}[\mathcal{L}_{test}(\Theta) - \mathcal{L}_{train}(\Theta)] \leq \mathcal{O}\left(\sqrt{\frac{VC(\mathcal{H})}{N}}\right)
\end{equation}

Substituting our bound on the VC-dimension:
\begin{equation}
\resizebox{0.49\textwidth}{!}{$
    \mathbb{E}[\mathcal{L}_{test}(\Theta) - \mathcal{L}_{train}(\Theta)] \leq \mathcal{O}\left(\sqrt{\frac{|\mathcal{R}| \cdot d^2 \cdot L \cdot \log(|\mathcal{R}| \cdot d^2 \cdot L)}{N}}\right)
    $}
\end{equation}

Since $d$ and $L$ are fixed hyperparameters of the model, and $|\mathcal{R}|$ is bounded by the KG size, we can simplify this to:
\begin{equation}
    \mathbb{E}[\mathcal{L}_{test}(\Theta) - \mathcal{L}_{train}(\Theta)] \leq \mathcal{O}\left(\sqrt{\frac{\log(|\mathcal{V}| \cdot |\mathcal{R}|)}{N}}\right)
\end{equation}

This completes the proof.
\end{proof}

\begin{theorem}[Inductive Generalization]
Let $\mathcal{G}_{train} = (\mathcal{V}_{train}, \mathcal{R}_{train}, \mathcal{F}_{train})$ and $\mathcal{G}_{test} = (\mathcal{V}_{test}, \mathcal{R}_{test}, \mathcal{F}_{test})$ be training and testing KGs. If the relation structures are similar, i.e., $d_{TV}(\mathcal{G}^{\mathcal{R}}_{train}, \mathcal{G}^{\mathcal{R}}_{test}) \leq \epsilon$, then the generalization error is bounded by:
\begin{equation}
    \mathcal{L}_{test}(\Theta) - \mathcal{L}_{train}(\Theta) \leq \mathcal{O}(\epsilon + \sqrt{\frac{\log(|\mathcal{V}_{train}| \cdot |\mathcal{R}_{train}|)}{N}})
\end{equation}
where $d_{TV}$ is the total variation distance between the relation graphs.
\end{theorem}

\begin{proof}
We decompose the generalization error into two components:
\begin{equation}
\resizebox{0.49\textwidth}{!}{$
    \mathcal{L}_{test}(\Theta) - \mathcal{L}_{train}(\Theta) = [\mathcal{L}_{test}(\Theta) - \mathcal{L}^*_{test}(\Theta)] + [\mathcal{L}^*_{test}(\Theta) - \mathcal{L}_{train}(\Theta)]
    $}
\end{equation}

where $\mathcal{L}^*_{test}(\Theta)$ is the expected loss under the optimal parameter setting for the test graph.

The first term represents the approximation error due to structural differences between train and test graphs, which is bounded by $\mathcal{O}(\epsilon)$ based on the similarity assumption.

The second term is the standard generalization error from the previous theorem.

Combining these bounds completes the proof.
\end{proof}

\begin{figure*}[ht]
  \centering
    \centering
    \includegraphics[width=\textwidth]{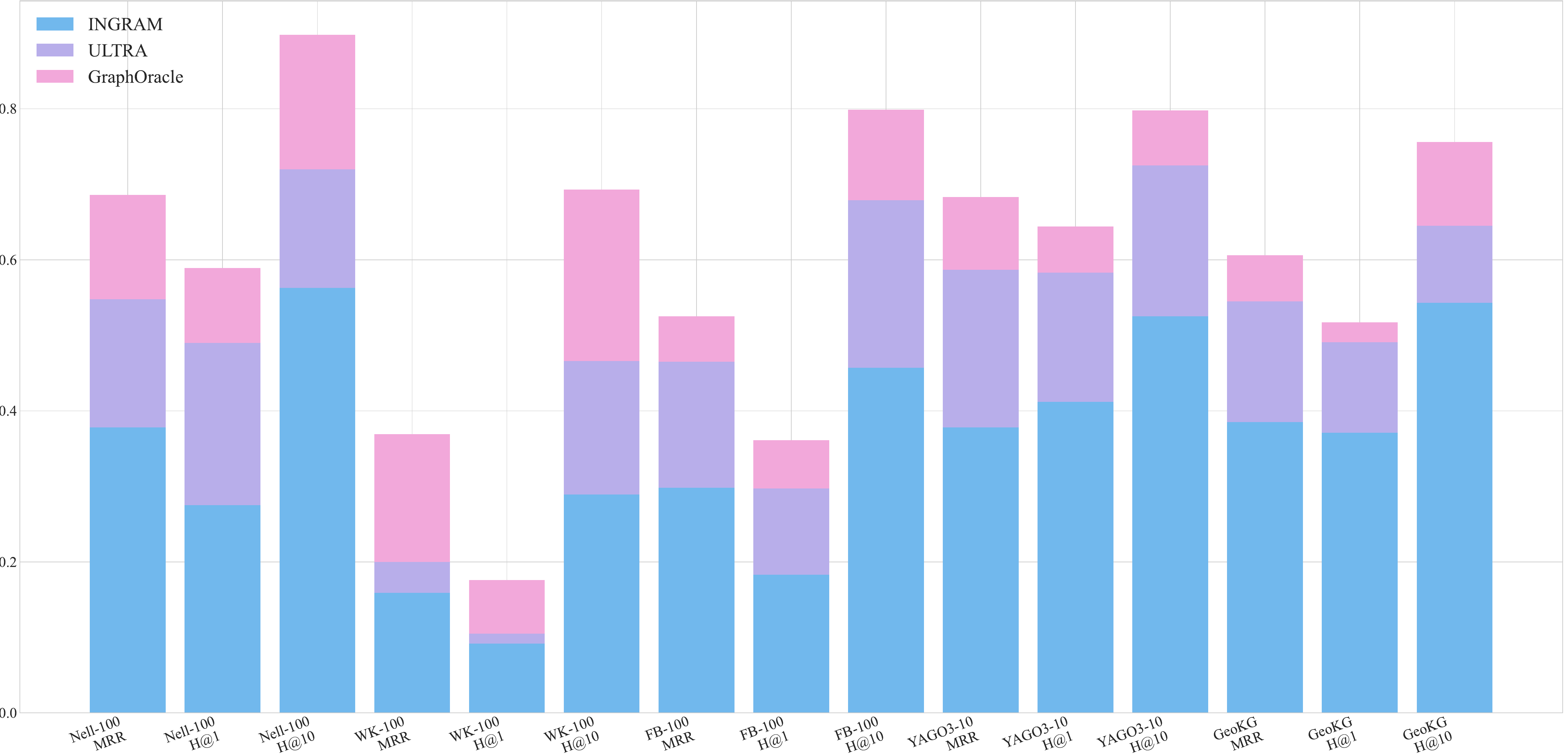}
  \caption{Comparison of the effects of building relationship graphs using different methods}
  \label{fig:different_relation_graph}
\end{figure*}

\section*{J  Comparison of \textsc{GraphOracle} with Supervised SOTA Methods} \label{Relationship diagram construction comparison}

\begin{table}[th]
\centering
\caption{The number of edges in the relation graph constructed by INGRAM, ULTRA, and \textsc{GraphOracle}.}
\label{tab:relgraph_edges}
\resizebox{0.5\textwidth}{!}{
\begin{tabular}{l|cccc}
\toprule
\textbf{Dataset} &\textbf{\# Relation} &\textbf{INGRAM} & \textbf{ULTRA} & \textbf{\textsc{GraphOracle}} \\
\midrule
NL-25  & 146  & 1610 & 2300 & 797 \\
NL-50  & 150  &  1748 & 2526 & 861 \\
NL-75  & 138  &  1626 & 2336 & 787 \\
NL-100 & 99  &   892 & 1159 & 416 \\
WK-25  & 67  &   598 &  947 & 256 \\
WK-50  & 102  &  1130 & 2164 & 508 \\
WK-75  & 77  &  732 & 1253 & 313 \\
WK-100 & 103  &  1052 & 1695 & 460 \\
FB-25  & 233  &  7172 & 10479 & 3501 \\
FB-50  & 228  &  6294 &  9300 & 3135 \\
FB-75  & 213  &  5042 &  7375 & 2524 \\
FB-100 & 202  &  4058 &  5728 & 2017 \\
\midrule
WN\_V1 & 9  &    48 &    40 &   37 \\
WN\_V2 & 10  &   76 &    76 &   55 \\
WN\_V3 & 11  &    94 &    85 &   68 \\
WN\_V4 & 9  &    70 &    61 &   54 \\
FB\_V1 & 180  &  1622 & 2416 &  712 \\
FB\_V2 & 200  &  2692 & 4050 & 1237 \\
FB\_V3 & 215  &  3398 & 5015 & 1640 \\
FB\_V4 & 219  &  4624 & 7036 & 2231 \\
NL\_V1 & 14  &   122 &  170 &   51 \\
NL\_V2 & 88  &  1574 & 2065 &  842 \\
NL\_V3 & 142  &  1942 & 2558 & 1017 \\
NL\_V4 & 76  &  1296 & 1657 &  744\\
\bottomrule
\end{tabular}
}
\label{number_of_edge}
\end{table}

\begin{figure*}[!ht]
  \centering
  \begin{subfigure}[b]{0.24\textwidth}
    \centering
    \includegraphics[width=\textwidth,height=2.9cm]{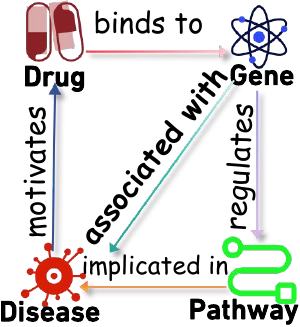}
    \caption{KG1}
  \end{subfigure}%
  \begin{subfigure}[b]{0.24\textwidth}
    \centering
    \includegraphics[width=\textwidth,height=2.9cm]{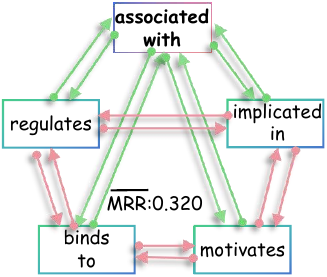}
    \caption{INGRAM}
  \end{subfigure}%
  \begin{subfigure}[b]{0.24\textwidth}
    \centering
    \includegraphics[width=\textwidth,height=2.9cm]{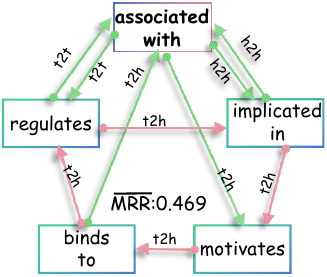}
    \caption{ULTRA}
    \label{fig:rel_ultra}
    \end{subfigure}%
  \begin{subfigure}[b]{0.24\textwidth}
    \centering
    \includegraphics[width=\textwidth,height=2.9cm]{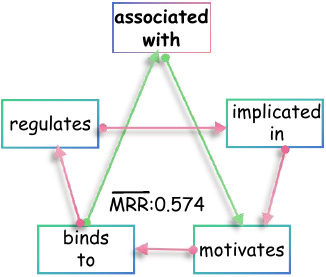}
    \caption{\textsc{GraphOracle}}
    \label{fig:rel_graphoracle}
    \end{subfigure}  \\

    \begin{subfigure}[b]{0.24\textwidth}
    \centering
    \includegraphics[width=\textwidth,height=2.9cm]{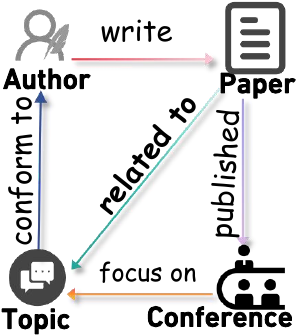}
    \caption{KG2}
  \end{subfigure}%
  \begin{subfigure}[b]{0.24\textwidth}
    \centering
    \includegraphics[width=\textwidth,height=2.9cm]{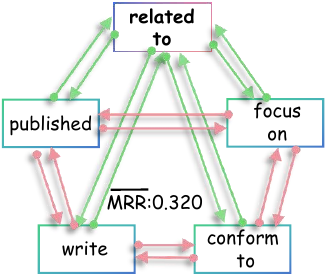}
    \caption{INGRAM}
  \end{subfigure}%
  \begin{subfigure}[b]{0.24\textwidth}
    \centering
    \includegraphics[width=\textwidth,height=2.9cm]{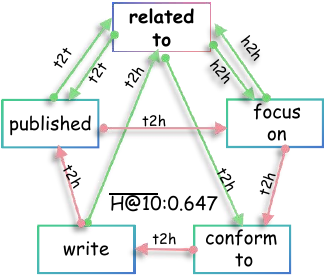}
    \caption{ULTRA}
    \end{subfigure}%
  \begin{subfigure}[b]{0.24\textwidth}
    \centering
    \includegraphics[width=\textwidth,height=2.9cm]{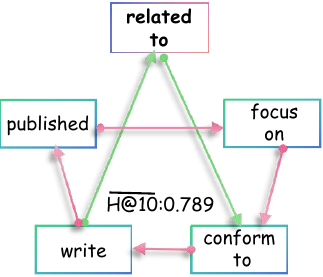}
    \caption{\textsc{GraphOracle}}
    \end{subfigure}  \\


  \caption{Comparison of \textsc{GraphOracle}'s graph construction with other methods}
  \label{fig:kg_relation_comparison2}
\end{figure*}

\begin{figure}[ht]
  \centering
  \begin{subfigure}[b]{0.45\textwidth}
    \centering
    \includegraphics[width=\textwidth,height=4cm]{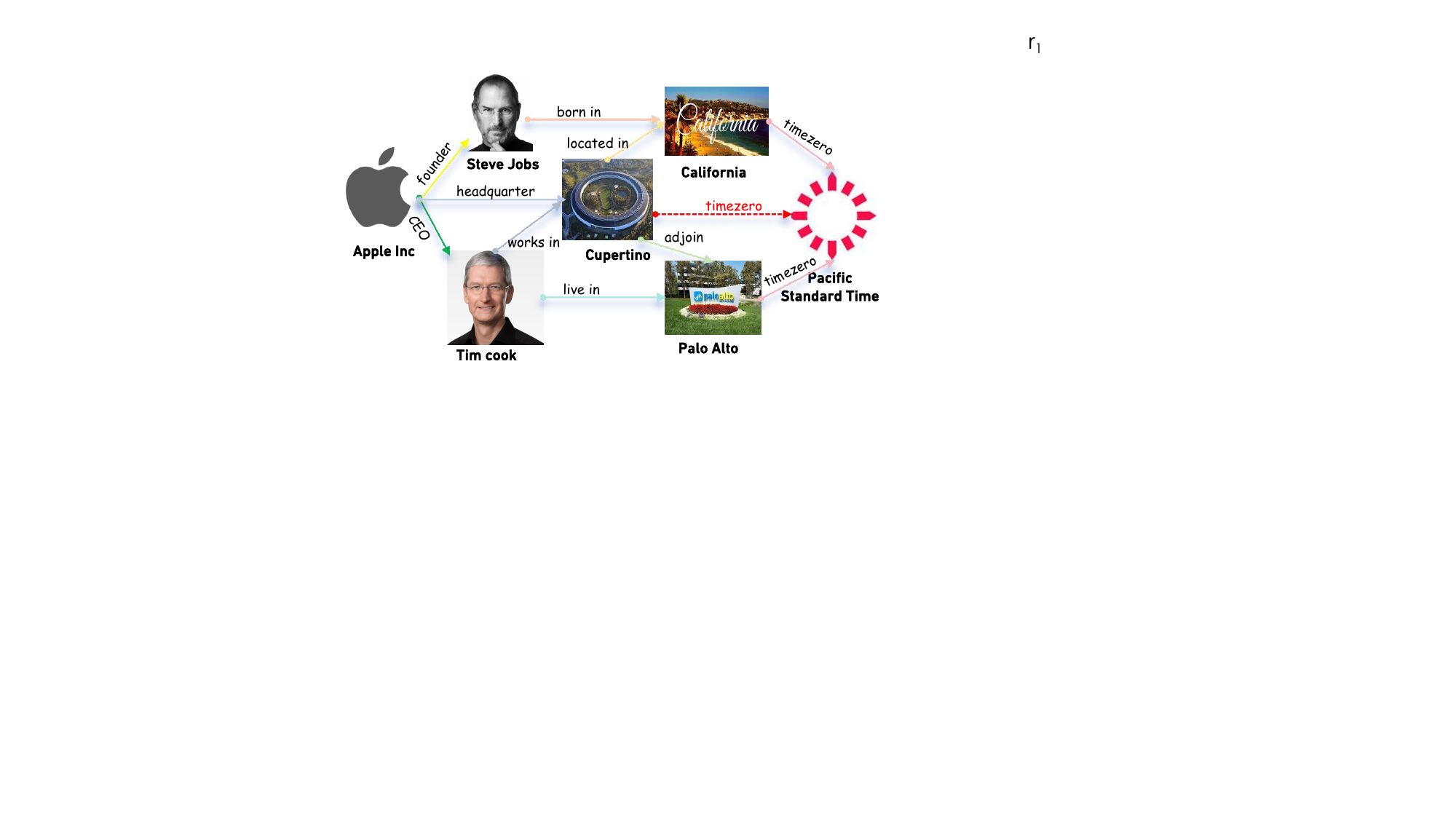}
    \caption{Knowledge Graph}
    \label{Fig:Knowledge Graph}
  \end{subfigure}%
\hspace{0.03\textwidth}
    \begin{subfigure}[b]{0.45\textwidth}
    \centering
    \includegraphics[width=\textwidth,height=4cm]{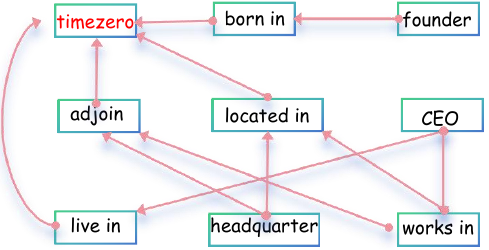}
    \caption{Relation Graph}
    \label{Fig:Relation Graph}
    \end{subfigure}%

  \caption{Illustration of GraphOracle's Relation-Dependency Graph Construction}
  \label{fig:kg_relation_comparison}
\end{figure}

Fig.~\ref{fig:kg_relation_comparison2} illustrates the comparative transfer learning performance (KG1 $\rightarrow$ KG2) of \textsc{GraphOracle}, INGRAM, and ULTRA\footnote{Fig.~\ref{fig:kg_relation_comparison} shows the visualization process of \textsc{GraphOracle} extracting RDG.}. Unlike INGRAM, which connects every pair of relations sharing an entity--thus producing an undirected $\mathcal{O}(|\mathcal{R}|^{2})$ co-occurrence graph with indiscriminate propagation that ignores directional dependencies--and ULTRA, which subdivides those links into four fixed head/tail interaction patterns (head-to-head, head-to-tail, tail-to-head, tail-to-tail) but still incurs quadratic growth, \textsc{GraphOracle} constructs a far sparser Relation-Dependency Graph (RDG) by keeping only directed precedence edges mined from two-hop relational motifs, reducing the edge count to $\mathcal{O}(|\mathcal{R}|\!\cdot\!\bar{d})$. These precedence edges impose an explicit partial order so that information propagates hierarchically from prerequisite to consequent relations, enabling the capture of high-order global dependencies that the local structures of its competitors overlook. Furthermore, while ULTRA directly applies NBFNet's method for entity and relation representation without specialized relation processing (limiting it to local relation structures), \textsc{GraphOracle} implements a query-conditioned multi-head attention mechanism that traverses the RDG to produce context-specific relation embeddings. This economical yet expressive design suppresses noise, lowers computational cost, and sustains both efficiency and accuracy as the relation set expands, explaining \textsc{GraphOracle}'s consistent superiority in transfer learning across real-world knowledge graphs.

We conducted a controlled experiment to isolate the impact of relation graph construction by replacing \textsc{GraphOracle}'s construction method with those of ULTRA and INGRAM, while maintaining identical message passing mechanisms. The results in Fig.~\ref{fig:different_relation_graph} demonstrate that \textsc{GraphOracle}'s relation-dependency graph construction yields consistently superior performance across all datasets and metrics. This empirically validates our theoretical claim that \textsc{GraphOracle} more effectively captures essential compositional relation patterns while filtering out spurious connections that introduce noise into the reasoning process. Notably, while the INGRAM and ULTRA graph construction variants underperform compared to \textsc{GraphOracle}, they still outperform their respective original message passing implementations—further confirming the effectiveness of our attention-based message propagation scheme. The efficiency advantage is quantitatively substantial: as shown in Table~\ref{number_of_edge}, \textsc{GraphOracle} generates significantly fewer edges (often 50-60\% fewer) than ULTRA and INGRAM across all benchmark datasets. This reduction in graph density translates directly to computational efficiency gains, with \textsc{GraphOracle} requiring proportionally less memory and computation during both training and inference phases. The performance improvements, coupled with this computational efficiency, demonstrate that \textsc{GraphOracle}'s approach to modeling relation dependencies fundamentally addresses the core challenge in knowledge graph foundation models: capturing meaningful compositional patterns without being overwhelmed by the combinatorial explosion of potential relation interactions.

\end{document}